\documentclass{article}

\usepackage[preprint]{neurips_2026}

\usepackage[utf8]{inputenc} 
\usepackage[T1]{fontenc}    
\usepackage{hyperref}       
\usepackage{url}            
\usepackage{booktabs}       
\usepackage{amsfonts}       
\usepackage{nicefrac}       
\usepackage{microtype}      
\usepackage{xcolor}         

\usepackage{xcolor}
\usepackage{amssymb}
\usepackage{arydshln}
\newcommand{\name}{\textbf{DiScGraph}}

\usepackage{pifont}
\usepackage{wrapfig}
\usepackage{makecell}
\usepackage{amssymb}
\usepackage{comment}
\usepackage{bm}
\usepackage{tablefootnote}
\usepackage{tabularx}
\usepackage{etoc}
\usepackage{float}

\usepackage{amsmath,amsfonts,bm}

\def\eqref#1{equation~\ref{#1}}

\def\1{\bm{1}}

\DeclareMathAlphabet{\mathsfit}{\encodingdefault}{\sfdefault}{m}{sl}
\SetMathAlphabet{\mathsfit}{bold}{\encodingdefault}{\sfdefault}{bx}{n}

\usepackage{url}
\usepackage{booktabs} 
\usepackage{graphicx}
\usepackage{adjustbox}
\usepackage{multirow}
\usepackage{amsthm}
\usepackage{enumitem}
\usepackage{soul}
\usepackage[ruled,vlined,algo2e]{algorithm2e}
\usepackage{tcolorbox}
\usepackage{wrapfig}
\usepackage{caption}
\newtheorem{theorem}{Theorem}[section]

\newtheorem{lemma}[theorem]{Lemma}

\newcommand*\circled[1]{\tikz[baseline=(char.base)]{%
    \node[shape=circle,fill=black,draw,inner sep=1pt,text=white] (char) {#1};}}

\tcbset{
  myhl/.style={
    colback=orange!20,
    colframe=orange!20,
    arc=3pt,        
    boxsep=1pt,     
    left=2pt, right=2pt, top=0pt, bottom=0pt,
    boxrule=0pt     
  }
}

\newtcbox{\myhl}{on line, myhl}

\tcbset{
  mehl/.style={
    colback=gray!20,
    colframe=gray!20,
    arc=3pt,        
    boxsep=1pt,     
    left=2pt, right=2pt, top=0pt, bottom=0pt,
    boxrule=0pt     
  }
}

\newtcbox{\mehl}{on line, mehl}

\usepackage[table]{xcolor}

\definecolor{lightgreen}{RGB}{220,245,220}
\definecolor{lightblue}{RGB}{220,235,255}

\newcommand{\best}[1]{\cellcolor{lightgreen}{#1}}
\newcommand{\second}[1]{\cellcolor{lightblue}{#1}}

\title{Faithful Image Generation via Discrete Diffusion \\ on Scene Graphs}
\title{Generating Scene Graphs via Discrete Diffusion}
\title{Dependency-Aware Discrete Diffusion for \\ Scene Graph Generation}

\author{%
  Rajalaxmi Rajagopalan\thanks{Corresponding author, <rr30@illinois.edu>} \\
  University of Illinois, Urbana-Champaign
\And
  Romit Roy Choudhury \\
  University of Illinois, Urbana-Champaign
}

\begin{document}

\maketitle

\begin{abstract}
Scene graphs (SGs) represent objects and their relationships as structured graphs, enabling applications in image generation, robotics, and 3D understanding. 
Recent work suggests that conditioning image generation on scene graphs improves compositional fidelity compared to text-only prompting. 
However, since users typically provide text rather than structured graphs, a key challenge is to generate scene graphs from natural language.
Prior work on discrete diffusion has demonstrated success in generating generic graphs such as molecules and circuits, but fails to account for the hierarchical structure and strong dependencies between objects, edges, and relations in scene graphs.
We address this limitation by introducing a dependency-aware, hierarchically constrained discrete diffusion model for scene graph generation. 
Our approach decouples structure and semantics across the forward and reverse processes, enabling the model to capture conditional dependencies. 
At inference time, we perform training-free conditioning to sample text-aligned scene graphs.
We evaluate our method on standard SG benchmarks and demonstrate improvements over both continuous and discrete graph generation baselines across graph and layout metrics. 
When fed to downstream image generation, our approach yields improved compositional alignment compared to text-to-image models, particularly in multi-object scenarios.
\end{abstract}

\section{Introduction}
Text-to-image (T2I) models have made remarkable progress in generating images from user-specified prompts, but user expectations are rising just as quickly. 
As prompts become longer and more complex, multiple lines of attack are being explored to improve compositional fidelity and alignment.
\begin{wrapfigure}{r}{0.47\textwidth}
    \centering
    \vspace{-0.15in}
    \includegraphics[width=0.98\linewidth]{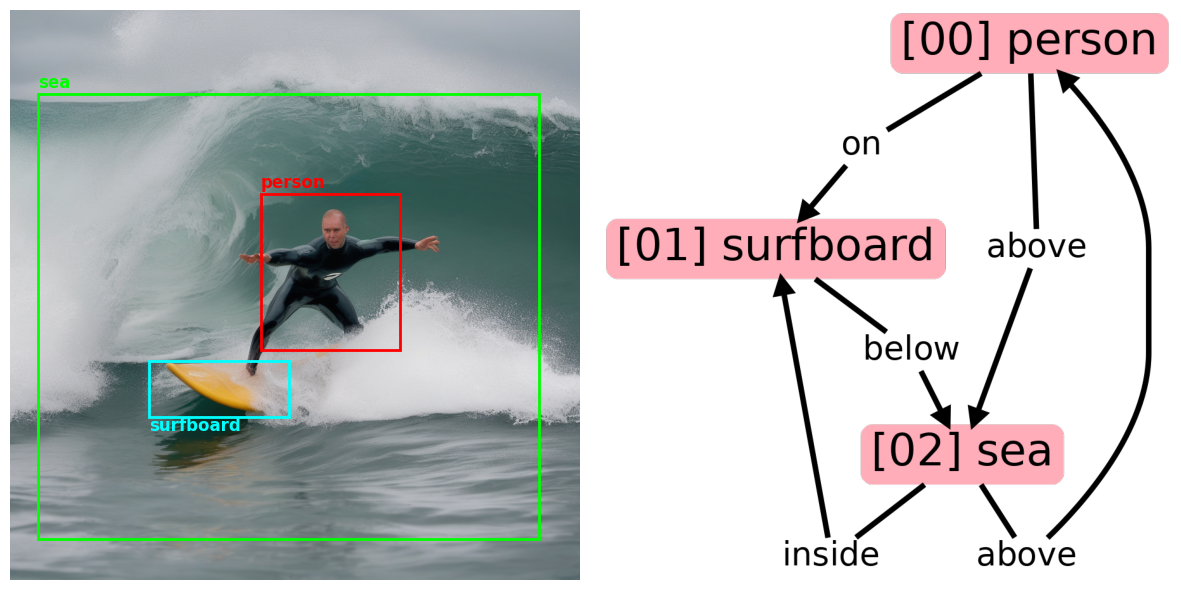}
    \caption{\small{Image $I$ and its scene graph $G$. Objects are pink nodes and relations are labels on edges.}}
    \vspace{-0.1in}
    \label{fig:sg_ex}
\end{wrapfigure}
One promising direction is to generate images from scene graphs (SGs)---see Fig. \ref{fig:sg_ex}---where visual objects are represented as nodes and their relationships as edges. 
Recent papers \cite{xu2024joint, farshad2023scenegenie, autoreg_sg2025, shen2025sgadapter} suggest that SG-conditioned image generation $p(I|G)$ can be more effective, particularly in compositional settings. 
While these gains are largely empirical, they are intuitively justified: text-based conditioning lacks explicit relational constraints, often leading to compositional inconsistencies. 
In contrast, scene graphs provide an explicit factorization over objects and relations, offering a structured prior that is difficult to encode in text embeddings.

Given this momentum, a natural question arises: how can scene graphs (SG) be generated from text prompts? 
While text-to-SG mapping may appear deterministic, it is inherently multi-modal due to the ambiguity and under-specification of natural language.
For example, a prompt such as “a person and a surfboard” does not specify their relationship; the person could be on, beside, or holding the surfboard.
Multiple valid SGs may correspond to the same prompt, necessitating modeling a conditional distribution $p(G|T)$ rather than predicting a single deterministic output.

Motivated by this, we propose to learn the manifold of scene graphs using discrete diffusion models, and then perform conditional generation of SGs from text prompts. 
However, this presents a fundamental challenge: scene graphs are significantly more complex than generic graphs.
They exhibit long-tailed vocabularies, strong conditional dependencies between objects and relations, and global structural constraints. 
Injecting these properties into discrete diffusion remains underexplored.

A key limitation of existing discrete diffusion approaches\cite{vignac2023digress}---designed for generic graphs like molecules or traffic networks---is that they rely on parameterizations that treat nodes and edges as independent categorical variables in both the forward and reverse processes. 
With scene graphs, such independence assumptions break due to the inherent conditional dependencies, such as the hierarchical structure where objects determine edges and edges determine relations.
Our key insight is that scene graphs require a \emph{hierarchically constrained discrete diffusion process}, where structure and semantics are modeled with explicit conditional dependencies rather than independent categorical transitions. 

We instantiate this insight by designing \circled{1} a hierarchical graph representation $G(V,E,R^{+})$ in which relations are defined only on active edges (i.e., $e_{ij} = 0\implies r_{ij} = 0$), 
\circled{2} an edge-gated forward noising process that respects this structure, and \circled{3} a factorized reverse sampler that generates objects, edges, and relations in a dependency-aware sequence: 
\begin{displaymath}
p_\theta(\hat{V_0},\hat{E_0},\hat{R_0}\mid x_t)
= p_\theta(\hat{V_0}\mid x_t) p_\theta(\hat{E_0}\mid \hat{V_0},x_t) p_\theta(\hat{R_0} \mid \hat{E_0},\hat{V_0},x_t) 
\end{displaymath}
Finally, \circled{4} at inference time, we perform training-free conditioning via reward-tilted sampling, enabling the generation of SGs aligned with the input text, $p(G|T)$.
We employ particle resampling (Sequential Monte Carlo): \( p_\theta(G_{t-1}| G_t;T) \propto p_\theta(G_{t-1}| G_t) \exp(\beta R(\hat{G}_0, T))\), with CLIP's similarity metric \cite{radford2021learning} to compute the reward $R(\hat{G}_0, T)$.
In sum, our main contribution lies in replacing standard independent categorical diffusion with a dependency-aware hierarchical diffusion process tailored to the unique requirements of scene graphs.
Fig. \ref{fig:overview} illustrates our operation pipeline.

We implement our model {\name} using a Graph Transformer \cite{dwivedi2021generalizationtransformernetworksgraphs} as the denoising network and train on standard scene graph datasets including Visual Genome \cite{krishna2016visualgenomeconnectinglanguage}, COCO\cite{lin2015microsoftcococommonobjects}, and LAION-SG \cite{li2024laionsgenhancedlargescaledataset}. 
We evaluate our approach on (i) graph metrics such as entity-MMD, Triplet-TV, and Rare-K TV, (ii) grounding metrics based on layout IoU, and (iii) image generation metrics including ImageReward, BLIP-VQA, and SG-IoU. 
Baselines include graph generation (DiGress \cite{vignac2023digress}, DiffuseSG \cite{xu2024joint}, etc.), layout models (LayoutDM \cite{lian2024llmgroundeddiffusionenhancingprompt}, LLMLayout-T2I \cite{layoutllm}), and text-based compositional image generators (CO3 \cite{dutta2026steerawaymodecollisions} and ComposeDiffusion \cite{compose-diff}) on COCO-validation and CompSGBench \cite{li2024laionsgenhancedlargescaledataset} benchmarks. 
Across these benchmarks, our method outperforms both continuous SG generation models and prior discrete graph baselines while achieving competitive performance on layout prediction.
Finally, when the generated scene graphs are used to condition downstream image generation models, we observe improved alignment compared to text-only T2I systems, particularly in compositional and multi-object scenarios.

\begin{figure*}
    \centering
    \includegraphics[width=1.05\linewidth]{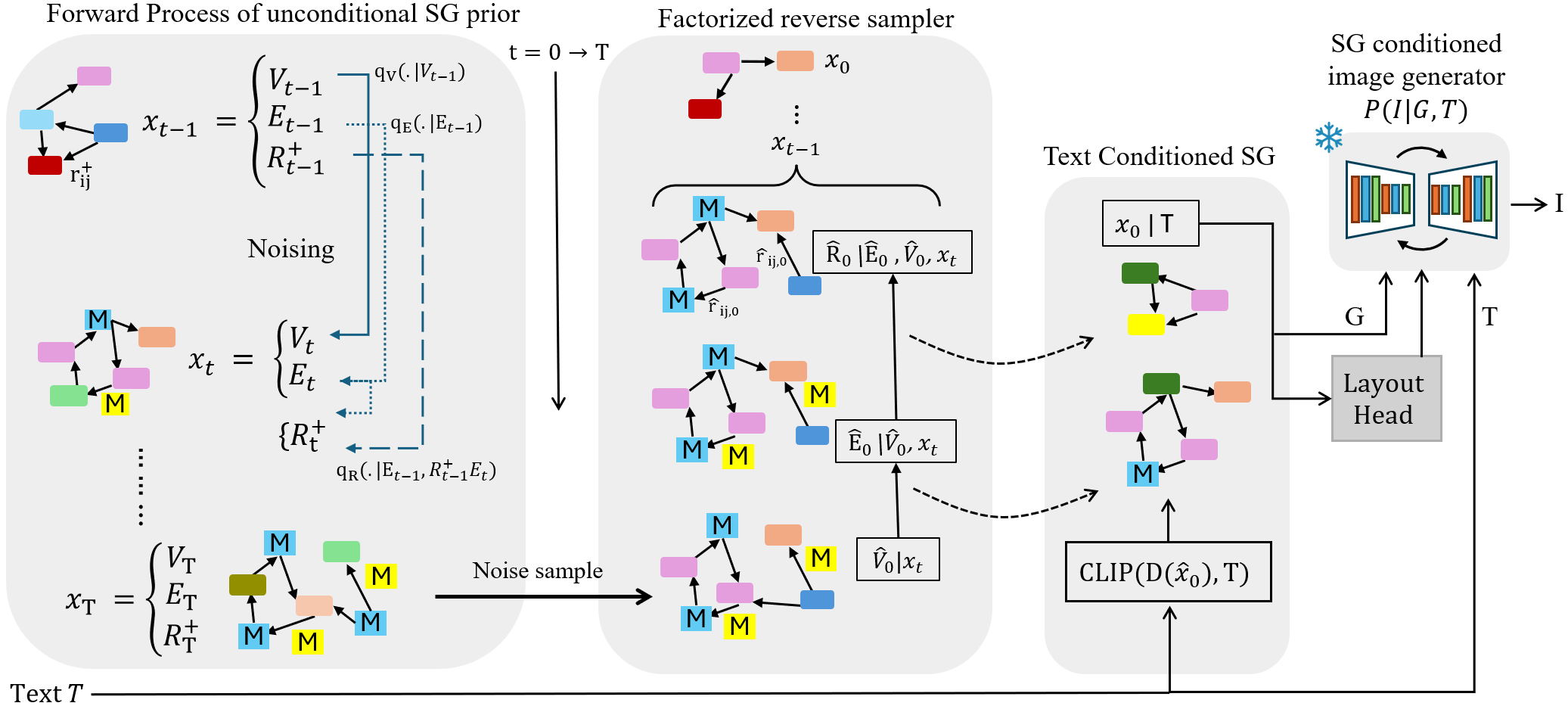}
    \caption{\small{{\name} pipeline: (1) Forward noising process in which relations are edge-gated. (2) Factorized reverse sampling that generates dependency-aware objects, edges and relations. (3) Reward-tilting at inference using the CLIP similarity with the text prompt as reward; a layout head also generates bounding boxes for graph nodes. (4) SG and text conditioned off-the-shelf image generator.}}    
    \vspace{-0.15in}
    \label{fig:overview}
\end{figure*}

\section{Preliminaries}
\vspace{-0.1in}
\noindent\textbf{Discrete Diffusion.}
Continuous diffusion \cite{song2021score,pmlr-v37-sohl-dickstein15}, typically using a Gaussian noising process, is not well-suited for categorical or structured data such as graphs. This motivates diffusion directly in discrete state spaces.
In discrete diffusion \cite{austin2023structureddenoisingdiffusionmodels}, each entity takes values in a finite set of $K$ categories. The forward process is defined via transition matrices $\{Q_t\}_{t=1}^T$, where, $q(x_t \mid x_{t-1}) = x_{t-1} Q_t,$
and the marginal at time $t$ is given by, $q(x_t \mid x_0) = x_0 \bar{Q}_t, \quad \bar{Q}_t = \prod_{s=1}^t Q_s.$
A common choice of $Q_t$ yields a uniform stationary distribution. The posterior used for training admits a closed form,
$q(x_{t-1} \mid x_t, x_0) \propto \left(x_t Q_t^\top\right) \odot \left(x_0 \bar{Q}_{t-1}\right).$

\noindent\textbf{Discrete Diffusion on Graphs.}
DiGress~\cite{vignac2023digress} extends discrete diffusion to graphs by modeling node and edge attributes as independent categorical variables. A graph is represented as $G = (V, E)$, and the forward process is defined by, $q(G_t \mid G_{t-1}) = \big(V_{t-1} Q^V_t,\; E_{t-1} Q^E_t\big),$
with closed-form marginals, $q(G_t \mid G_0) = \big(V_0 \bar{Q}^V_t,\; E_0 \bar{Q}^E_t\big).$
The reverse model predicts clean node and edge distributions, and sampling proceeds independently across nodes and edges,
\begin{equation}
p_\theta(G_{t-1} \mid G_t)
= \prod_{i} p_\theta(v^i_{t-1} \mid G_t)
  \prod_{i,j} p_\theta(e^{ij}_{t-1} \mid G_t).
\end{equation}
Each term is computed by marginalizing over predicted clean states (nodes \& edges),
\begin{equation}
p_\theta(v^i_{t-1} \mid G_t)
= \sum_{v_0} q(v^i_{t-1} \mid v_0, x^i_t)\, \hat{p}^i_V(v_0), \qquad
p_\theta(v^i_{t-1} \mid v_0, G_t) =
\begin{cases}
q(v^i_{t-1} \mid v_0, x^i_t), & \text{if } q(x^i_t \mid v_0) > 0, \\
0, & \text{otherwise}.
\end{cases}
\end{equation}
This formulation enforces permutation invariance, and treats nodes and edges as independent categorical variables. However, it assumes a categorical space where edge existence and semantics, edge identity are coupled, which is unsuitable for scene graphs' strong inter-entity dependence, large relation vocabularies, and long-tailed relation distribution.



\section{Method}
\vspace{-0.05in}
{\name}'s redesign of discrete diffusion entails $4$ parts, starting with a (1) hierarchical representation of SGs, (2) an edge-gated forward noising process (3) a factorized reverse sampler for unconditional generation, and (4) a text-conditioned SG generation at inference time (see Fig. \ref{fig:overview}).
\vspace{-0.05in}

\subsection{Hierarchical Scene Graph Representation}
\vspace{-0.1in}
In discrete diffusion, a graph is represented as a flat categorical state $G=(V,E)$, where node labels $V$ and edge labels $E$ are independently modeled \cite{vignac2023digress}. 
While this representation is effective for, say, molecules or transportation graphs (where edge relations are defined over a small, well-defined vocabulary with strict constraints), SGs involve large, long-tailed relational vocabularies and strong object–relation dependencies.
Modeling the edges with a single categorical variable, say $e_{ij} \in \{noEdge, r_1, r_2, \dots, r_K \}$, biases the learning towards edges rather than relations, especially for sparse scene graphs, large $K$, and the long-tailed distribution over relation classes $r_{k}$. 
Scene graphs are better modeled by separating edges (structure) from the relations (semantics), so that \emph{relations can be conditioned only on active edges}, i.e., $e_{ij}=0 \Rightarrow r_{ij}=0$.
This motivates our decoupled representation in which each edge is factorized into two: (i) edge existence, $e_{ij} \in \{0,1\}$, and (ii) relation identity, $r^+_{ij} \in \{1,\dots,K_{rel}\}$.
We define valid scene graph space as:
\begin{equation}
\mathcal{G}_{SG} = \{(V,E,R^+) : e_{ij}=0 \Rightarrow r_{ij}=0\}, \quad ~~
\big(V,E,R^+ \big) = \big( (v_{i}), (e_{ij}), (r^+_{ij}) \big)_{i,j=1}^N
\end{equation}
\begin{align}
v_i \in &\mathcal{V} = \{1,\dots,K_{\text{obj}}\}, \quad 
e_{ij} \in \{0,1\}, \quad 
r^+_{ij} \in \{1,\dots,K_{\text{rel}}\}.
\end{align}

Thus the abstract graph state space $\mathcal{G} := \mathcal{V}^N \times \{0,\dots,K_{rel}\}^{N\times N}$ is modified to the constrained scene graph space, $\mathcal{SG} := \mathcal{V}^N \times \{0,1\}^{N\times N} \times \{1,\dots,K_{rel}\}^{N\times N}$.
The factorized state explicitly encodes:
(i) directedness,
(ii) sparsity via $E$,
(iii) semantic validity via $R^+$.
This representation underlies all subsequent modeling components.

\subsection{Structure-Aware Forward Process}

We now define a discrete diffusion process over the factorized state as,
\begin{equation}
q(x_{1:T} \mid x_0) = \prod_{t=1}^T q(x_t \mid x_{t-1}), \quad x_t = (V_t, E_t, R_t^+).
\end{equation}
The forward corruption is redesigned for the factorized transition over the scene graph space $\mathcal{SG}$ as defined below, where $V$ and $E$ corruptions follow standard discrete diffusion forward process \cite{austin2023structureddenoisingdiffusionmodels, vignac2023digress}:
\begin{equation}
q(x_t \mid x_{t-1}) =
q_V(V_t \mid V_{t-1})\,
q_E(E_t \mid E_{t-1})\,
q_R(R_t^+ \mid R_{t-1}^+, E_t, E_{t-1}).
\end{equation} 
\textbf{Relation corruption (edge-aware).}
Relations are only defined on active edges, so corruption must depend on edge state, as follows,
\begin{equation}
q_R(R_t \mid R_{t-1}, E_t)
=
\prod_{i \neq j}
q_R(r_{ij,t} \mid r_{ij,t-1}, e_{ij,t}).
\end{equation}
\vspace{-0.08in}
\begin{equation}
q_R(r_{ij,t} \mid r_{ij,t-1}, e_{ij,t}) =
\begin{cases}
\delta_0(r_{ij,t}), & e_{ij,t} = 0, \\
\bar{q}_R(r_{ij,t} \mid r_{ij,t-1}), & e_{ij,t} = 1,
\end{cases}
\end{equation}
\vspace{-0.05in}
For active edges ($e_{ij}=1$) with relation prior $\pi_R$\footnote{We slightly abuse notations and use $R_t$ and $R_t^+$ interchangeably, but it implies relations on active edges.},
\begin{equation}
\Tilde{q}_R(r_{ij,t} \mid r_{ij,t-1})
=
(1-\beta_t^R)\delta_{r_{ij,t-1}}
+
\beta_t^R \pi_R,
\end{equation}
\begin{equation}
q_R(R_t \mid R_{t-1}, E_t)
=
\prod_{ij:e_{ij,t}=1} \Tilde{q}_R(r_{ij,t} \mid r_{ij,t-1})
\prod_{ij:e_{ij,t}=0} \delta_0(r_{ij,t}).
\end{equation}
\vspace{-0.05in}
Hence, the combined corruption process is:
\small{
\begin{equation}
\begin{aligned}
q(x_t \mid x_{t-1})
&=
\prod_i q_V(v_{i,t} \mid v_{i,t-1})
\prod_{i\neq j} q_E(e_{ij,t} \mid e_{ij,t-1})
\prod_{ij:e_{ij,t}=1} \tilde{q}_R(r_{ij,t} \mid r_{ij,t-1})
\prod_{ij:e_{ij,t}=0} \delta_0(r_{ij,t})
\end{aligned}
\end{equation}
}
 We derive the stationary distribution for the factorized discrete state in the Appendix Sec. \ref{app:fwd_conv} and provide closed-form expressions for the marginals $q(x_t|x_0)$ in the Appendix Sec. \ref{app:fwd_marginal}. 
In abstract graph works, the categorical edge prior is $\pi_E = \{\pi_E(0),\pi_E(1),\dots,\pi_E(K_{rel})\}$, but due to sparsity, large $K_{rel}$, and long-tailed relation distribution in scene graphs, $\pi_E(0) \gg \pi_E(r), r \in \{1,\dots,K_{rel}\}$. 
Thus, without factorization, under random corruption, many relations are pushed towards "no-edge". This implies that rare relations compete against both frequent relations and the overwhelmingly frequent no-edge class. In contrast, the factorized corruption happens over $\{1,\dots,K_{rel}\}$ instead of $\{0,1,\dots,K_{rel}\}$. Thus, rare relations no longer compete with edge absence during relation denoising, resulting in better relation modeling.

In addition to random corruption, we also apply absorbing mask corruption on $V$ and $R^+$, which preserves the SG structure and only hides the semantics information, allowing the model to learn SG semantics. 
For mask tokens $[\mathrm{MASK}]_V$ and $[\mathrm{MASK}]_R$, we employ a hybrid random + mask corruption strategy. 
The hybrid corruption kernel for objects is as follows (with more details in the Appendix Sec. \ref{app:fwd_mask}):
\begin{equation}
q_V(v_{i,t} \mid v_{i,t-1})
=
(1-\beta_t^V)\delta_{v_{i,t-1}}
+
\beta_t^V\big[
\rho_V \delta_{[\mathrm{MASK}]_V}
+
(1-\rho_V)\pi_V
\big].
\end{equation}

\subsection{Factorized Reverse Sampler}
\label{sec:rev_sampler}

Given the constrained scene graph state at time $t$, $\{x_t = (V_t,E_t,R_t^+),~ e_{ij,t}=0 \Rightarrow r_{ij,t}=0\}$, 
we need to sample $x_{t-1} \sim p_\theta(x_{t-1}\mid x_t)$ 
while preserving this constraint. As in discrete diffusion, we parameterize the reverse process through a denoiser $\epsilon_{\theta}$ prediction of the clean state \(x_0\), and combine this prediction with the known forward kernels to obtain a posterior over \(x_{t-1}\).
For scene graphs, the clean state is $x_0=(V_0,E_0,R_0).$
Following the structured state space, we factorize the clean-state prediction as
\begin{equation} \begin{aligned}
p_\theta(x_0\mid x_t)
=
p_\theta(V_0,E_0,R_0\mid x_t) 
=
p_\theta(V_0\mid x_t)\,
p_\theta(E_0\mid V_0,x_t)\,
p_\theta(R_0\mid V_0,E_0,x_t).
\end{aligned} 
\label{eqn:rev_sampler}
\end{equation} 
This factorization follows the semantic structure of scene graphs, where:
(1) first, object identities define the semantic entities in the graph, then
(2) directed edge existence determines which object pairs interact, and then 
(3) relation identities depend on objects and are meaningful only for active directed edges.

The combined object, edge, and relation posterior updates give the structured reverse transition is,
\begin{equation} \begin{aligned}
p_\theta(x_{t-1}\mid x_t)
&=
p_\theta(V_{t-1}\mid x_t)\,
p_\theta(E_{t-1}\mid V_{t-1},x_t)\,
p_\theta(R_{t-1}\mid V_{t-1},E_{t-1},x_t).\\
&=
\prod_i p_\theta(v_{i,t-1}\mid v_{i,t},x_t)
\prod_{i\neq j}p_\theta(e_{ij,t-1}\mid e_{ij,t},V_{t-1},x_t)
\prod_{i\neq j}p_\theta(r_{ij,t-1}\mid r_{ij,t},e_{ij,t-1},V_{t-1},E_{t-1},x_t).
\end{aligned} \end{equation} 
{\name} redesigns the relation posterior to be compatible with the factorized state space $\mathcal{SG}$. This opens up three scenarios: (1) $e_{t-1}$ does not exist, (2) both $e_{t-1}$ and $e_t$ exist, (3) $e_{t-1}$ has come to exist,
\begin{equation} \begin{aligned}
p_\theta(r_{ij,t-1}\mid r_{ij,t},e_{ij,t-1},e_{ij,t},x_t)
=
\begin{cases}
\delta_0(r_{ij,t-1}),
& e_{ij,t-1}=0, \\[6pt]
p_\theta^{\mathrm{post}}(r_{ij,t-1}\mid r_{ij,t},x_t),
& e_{ij,t-1}=1,\ e_{ij,t}=1, \\[6pt]
p_\theta^{\mathrm{marg}}(r_{ij,t-1}\mid x_t),
& e_{ij,t-1}=1,\ e_{ij,t}=0.
\end{cases}
\end{aligned} \end{equation}
where $p_\theta^{\mathrm{post}}(r_{ij,t-1}\mid r_{ij,t},x_t)$ denotes the standard relation posterior distribution and $p_\theta^{\mathrm{marg}}(r_{ij,t-1}\mid x_t)$ denotes the marginal distribution obtained by moving the denoiser's clean prediction to time $t-1$. Please refer to Appendix Sec. \ref{app:rev_sample} for the detailed derivation of the posterior.
Thus, the factorized reverse sampler is the reverse-time counterpart of the structure-aware forward process. It uses the known discrete posterior for categorical diffusion, but modifies appropriately with \emph{edge-gating}, needed for our constrained scene graph state. 

In addition to the base sampler, we also explore three types of sampler refinements for robust sampling, such as (1) Split Gibbs refinement \cite{Vono_2019,chu2026split}, (2) Soft masked refinement, and (3) Rare relation refinement for improving semantic consistency, correcting uncertain predictions, and improving rare relation generation, respectively. 
Appendix Sec. \ref{app:sampler_refinements} provides details on each strategy and their schedules.

\subsection{Text-Conditioned Scene Graph Generation via Reward-Tilted Inference}
\label{app:text_tilting}
The structured diffusion model described above defines an unconditional generative prior over scene graphs, $p_\theta(G_0)$, $G_0$ is the clean sampler output. We aim to generate scene graphs conditioned on a text prompt \(T\), i.e., $P(G_0 \mid T).$
Instead of training a conditional model \(p_\theta(G_0 \mid T)\), we perform \emph{training-free alignment} by sampling from a reward-tilted distribution, following works like \cite{kim2025test,yeh2024training},
\begin{equation} \begin{aligned}
\tilde p(G_0 \mid T)
\propto
p_\theta(G_0)\,\exp\big(\beta R(G_0, T)\big),
\end{aligned} \label{eqn:reward-tilt}\end{equation}
where \(R(G_0,T)\) is a CLIP \cite{radford2021learning} similarity score between the scene graph and the text prompt, and \(\beta > 0\) controls the strength of conditioning (Appendix Sec. \ref{app:clip} gives more information on the reward).
The reward depends only on the final state \(G_0\), while sampling proceeds sequentially from \(t=T\) to \(t=0\). So we can define intermediate timestep reward approximations $R_t := R(\hat G_0^{(t)}, T)$ using the denoiser's $\epsilon_\theta$ predictions of the clean graph at $t$, $\hat G_0^{(t)} = \epsilon_\theta(G_t, t)$.\\
\textbf{Tilted Sampling.}
Ideally, we would sample from the tilted distribution,
\[
\tilde p_\theta(G_{t-1} \mid G_t, T)
\propto
p_\theta(G_{t-1} \mid G_t)\,
\exp\big(\beta R(\hat G_0(G_{t-1}), T)\big),
\]
but this distribution is intractable, so we use Sequential Monte Carlo sampling (See \cite{kim2025test} and Appendix Sec. \ref{app:smc}).
Thus, by leveraging the intermediate clean-state predictions \(\hat G_0\mid G_t\), we can guide the reverse diffusion trajectory toward samples that are aligned with the input text, while generating valid scene graphs.

\section{Experiments}
\vspace{-0.1in}

$\blacksquare$ \textbf{Datasets}:
We train the scene graph generator on popular scene graph datasets:\myhl{Visual Genome (VG)}\cite{krishna2016visualgenomeconnectinglanguage}, \myhl{COCO-Stuff}\cite{lin2015microsoftcococommonobjects}, and a large-scale dataset \myhl{LAION-SG}\cite{li2024laionsgenhancedlargescaledataset}. 
Both quantitative and qualitative results reported are from evaluations on \myhl{VG/COCO-validation} and \myhl{CompSGBench}\cite{li2024laionsgenhancedlargescaledataset} benchmark, consisting of SGs describing complex multi-object-relation visual scenes. 

\vspace{-0.02in}
$\blacksquare$ \textbf{Image Generators}: For T2I generation task, we use SG-to-image generator \myhl{SDXL-SG} \cite{li2024laionsgenhancedlargescaledataset}, and Layout-to-image generator, \myhl{GLIGEN}\cite{li2023gligenopensetgroundedtexttoimage}.

\vspace{-0.02in}
$\blacksquare$ \textbf{Baselines}:\mehl{(1) \emph{Graph Generators}:}We compare generated unconditional SGs against generic diffusion graph generators trained on SG datasets: \myhl{DiGress}\cite{vignac2023digress}, \myhl{GraphGDP}\cite{huang2022graphgdp}, \myhl{D3PM}\cite{austin2023structureddenoisingdiffusionmodels}, autoregressive SG generators: \myhl{SceneGraphGen}\cite{garg2021unconditional} and \myhl{VarScene}\cite{verma2022varscene}, and diffusion SG work: \myhl{DiffuseSG}\cite{xu2024joint}.
\mehl{(2) \emph{Layout Generators}:} We compare layouts extracted from {\name}'s layout head against Transformer model:\myhl{BLT}\cite{kong2022blt}, discrete diffusion works:\myhl{LayoutDiffusion}\cite{zhou2023layoutdiffusion}, \myhl{LayoutDM}\cite{inoue2023layoutdm} that generate discretized layouts trained on layouts from SG datasets, and \myhl{DiffuseSG}\cite{xu2024joint} that generates layouts along with scene graphs. 
\mehl{(3) \emph{T2I Image Generation}:} We compare text-conditioned SG to image generation performance against text-only compositional works:\myhl{ComposedDiff}\cite{compose-diff}, \myhl{CO3}\cite{dutta2026steerawaymodecollisions} on SDXL, and base \myhl{SDXL}\cite{sdxl}. We also compare layout-to-image generation with LLM augmented works: \myhl{LayoutLLM-T2I}\cite{layoutllm}, and \myhl{LDM}\cite{lian2024llmgroundeddiffusionenhancingprompt}.

\vspace{-0.02in}
$\blacksquare$ \textbf{Metrics}: We use 
\mehl{(1) \emph{Unconditional Graph Metrics}:} \myhl{N-MMD}, \myhl{R-MMD}, \myhl{ID-MMD}, \myhl{OD-MMD} \cite{xu2024joint} (for baselines, relations and edges are synonymous). 
A triplet is defined as <\texttt{object,relation,subject}>, and\myhl{Triplet-TV} measures if semantic consistency is preserved. \myhl{Attach-TV (AT-TV)} measures how accurately the model generates relations that "attach" to given objects. \myhl{RARE-K-TV (R-K-TV)} measures the model's capability to generate long-tailed relations.
\mehl{(2) \emph{Layout Metrics}:} \myhl{F1 scores} between generated and ground truth layouts.  
\mehl{(3) \emph{Image Quality Metrics}:} \myhl{FID}, \myhl{CLIP-I2I}, \myhl{CLIP-I2T} \cite{radford2021learning}, \myhl{BLIP-VQA} \cite{t2icompbench}, and \myhl{ImageReward (IR)} \cite{xu2023imagereward}. 
\mehl{(4) Scene-graph-to-image metrics:} \myhl{SG-IoU}, \myhl{Entity-IoU}, \myhl{Relation-IoU} \cite{shen2024sgadapterenhancingtexttoimagegeneration}, that measure structural overlap between generated and real images; quantified via IoU of SGs, objects, and relations extracted by an LLM from the image pairs. \\
$\blacksquare$ \textbf{Training Details}: 
We use the graph transformer \cite{yun2020graphtransformernetworks} backbone for the denoiser $\epsilon_\theta$, following DiGress \cite{vignac2023digress}. However, the embeddings and denoiser prediction heads are modified for the SG factorized formulation. Additionally, we also use a layout head that extracts layouts from sampled scene graphs. GIoU layout loss ensures a geometrically consistent scene graph generator. 
Please refer to Appendix Sec. \ref{app:experiment} for more details on metrics and implementation details.

\subsection{Results}
\vspace{-0.05in}
{\name}'s two main performance goals are: (1) scene graph generation from a learnt structural prior $P(G)$, and (2) improving text-to-image (T2I) generation with text-conditioned scene graphs. 

\circled{1} \textbf{Scene Graph Generation} \\
\emph{Unconditional scene graphs}:
Table \ref{tab:graph_metrics} reports results on $N=1000$ generated SGs. The reported distributional metrics, MMD and TV, analyze their statistics against ground truth. 
The corresponding qualitative scene graphs are visualized in Figures \ref{fig:uncond_vg}, \ref{fig:uncond_vg_1}, \ref{fig:uncond_coco}, \ref{fig:uncond_comp}. 
Results show that {\name} achieves superior performance across all graph metrics. 
Particularly, the factorized formulation boosts relation-learning significantly, as evidenced by the large gap between {\name} and the second-best baselines (DiffuseSG and DiGress) on Relation-MMD (R-MMD) and Attach-TV (AT-TV). 
R-MMD quantifies {\name}'s ability at learning true relation distribution, while AT-TV specifically measures semantic consistency via $p(R|V)$ i.e., given $V$, can a model predict right $R$. 

\vspace{-0.05in}

\begin{table*}[hbt!]
\centering
\caption{Comparing graph methods on sampled graph metrics on val/test datasets (\colorbox{lightgreen}{Best}, \colorbox{lightblue}{2nd best}).}

\setlength{\tabcolsep}{2.5pt}
\renewcommand{\arraystretch}{1.2}

\fontsize{8.5}{9.5}\selectfont
\begin{tabular}{|c|c|c|c|c|c|c|c|}
\hline
Method & N-MMD ($\downarrow$) & R-MMD ($\downarrow$) & ID-MMD ($\downarrow$) & OD-MMD ($\downarrow$) & TRIP-TV ($\downarrow$) & R-K-TV ($\downarrow$) & AT-TV ($\downarrow$)\\
\hline

\multicolumn{8}{|c|}{\textbf{Visual Genome}} \\
\hline

GraphGDP & $1.3e^{-2}$ & $3.40e^{-2}$ & $8.67e^{-2}$ & $9.35e^{-2}$ & 0.8980 & 0.999 & 0.5904 \\
D3PM\footnotemark & $7.69e^{-3}$ & $2.00e^{-2}$ & $3.07e^{-2}$ & $3.07e^{-2}$ & 0.8160 & - & - \\
DiGress & \second{$7.63e^{-3}$} & \second{$7.26e^{-3}$} & $9.27e^{-3}$ & $9.66e^{-3}$ & 0.7730 & 0.8622 & 0.4561 \\
\hline

SceneGraphGen & $1.28e^{-2}$ & $9.66e^{-3}$ & $1.54e^{-2}$ & $1.29e^{-2}$ & 0.9540 & 0.984 & 0.3264 \\
VarScene & $3.68e^{-2}$ & $2.55e^{-2}$ & $2.45e^{-2}$ & $2.51e^{-2}$ & 0.9880 & 0.997 & \second{0.2348} \\
DiffuseSG$_{\text{sg}}$ & $8.21e^{-3}$ & $1.11e^{-2}$ & \second{$7.05e^{-3}$} & \second{$6.01e^{-3}$} & \second{0.6390} & \second{0.826} & 0.2916 \\
{\name}$_{\text{sg}}$  & \best{$4.59e^{-3}$} & \best{$5.78e^{-3}$} & \best{$4.32e^{-3}$} & \best{$4.46e^{-3}$} & \best{0.4920} & \best{0.5194} & \best{0.1766}\\

\hline

\multicolumn{8}{|c|}{\textbf{COCO}} \\
\hline

GraphGDP & $1.55e^{-2}$ & $7.13e^{-3}$ & $5.00e^{-4}$ & $7.00e^{-4}$ & 0.7855 & 0.982 & 0.5388 \\
D3PM$^2$ & $4.92e^{-4}$ & $1.29e^{-4}$ & \best{$0$} & \best{$0$} & \second{0.305} & - & - \\
DiGress & $7.81e^{-4}$ & $8.86e^{-4}$ & $9.01e^{-6}$ & $1.95e^{-5}$ & 0.4147 & 0.855 & 0.3775 \\
\hline

SceneGraphGen & $6.06e^{-4}$ & $9.92e^{-5}$ & $9.00e^{-4}$ & $9.00e^{-4}$ & 0.9486 & 0.977 & \second{0.2822} \\
VarScene & $6.23e^{-2}$ & 0.104 & 0.1040 & 0.0994 & 0.8684 & 0.961 & 0.2902 \\
DiffuseSG$_{\text{sg}}$ & \best{$3.93e^{-4}$} & \second{$6.45e^{-5}$} & $8.47e^{-6}$ & $9.91e^{-6}$ & 0.3106 & \second{0.814} & 0.2916 \\
{\name}$_{\text{sg}}$  & \second{$3.97e^{-4}$} & \best{$5.39e^{-5}$} & \second{$6.37e^{-6}$} & \second{$5.76e^{-6}$} & \best{0.2100} & \best{0.6661} & \best{0.1204} \\

\hline

\multicolumn{8}{|c|}{\textbf{LAION-SG}} \\
\hline

DiGress & $3.42e^{-3}$ & $9.78e^{-3}$ & $7.35e^{-3}$ & $8.00e^{-3}$ & 0.8188 & 0.9164 & 0.5526 \\
\hline

DiffuseSG$_{\text{sg}}$ & \second{$1.36e^{-3}$} & \second{$9.66e^{-3}$} & \second{$6.50e^{-3}$} & \second{$6.90e^{-3}$} & \second{0.7025} & \second{0.8973} & \second{0.5440} \\
{\name}$_{\text{sg}}$  & \best{$9.81e^{-4}$} & \best{$6.35e^{-3}$} & \best{$5.90e^{-3}$} & \best{$5.80e^{-3}$} & \best{0.5693} & \best{0.7746} & \best{0.4238}\\

\hline
\end{tabular}

\label{tab:graph_metrics}
\end{table*}
\footnotetext{We are unable to replicate the results of D3PM reported in \cite{xu2024joint} so we report them as provided.}

\emph{Does {\name} learn rare relations when plausible object pairs are presented to it?}
Rare-K-TV evaluates this behavior and results show that by decoupling structure $E$ and semantics $R$, {\name} is uniquely suited for scene graphs' long-tailed relation distributions. 
Moreover, the In/Out Degree MMD metrics show that learning edge-existence structure is not compromised in favor of relations. 
The factorized sampler that predicts $V \to E \to R$ results in {\name}'s sampled graphs to adhere strongly to semantic consistency, thereby improving Triplet-TV (<\texttt{object,relation,subject}>). 
Thus, the factorized constrained discrete space and reverse sampler together ensure {\name} learns the language semantics (strong conditional dependencies) and the long-tailed vocabularies.
This makes {\name} a more robust SG generator over a continuous diffusion baseline (DiffuseSG), and over a generic graph discrete framework (DiGress).

\begin{table*}[t]
\centering

\setlength{\tabcolsep}{3pt}
\renewcommand{\arraystretch}{1.2}

\begin{minipage}[t]{0.45\textwidth}
\centering
\caption{\small{Comparing layout methods w/ layout head output (\colorbox{lightgreen}{Best}, \colorbox{lightblue}{2nd best}).}}
\fontsize{7}{8}\selectfont
\begin{tabular}{|c|c|c|c|c|}
\hline
\multicolumn{5}{|c|}{\textbf{Visual Genome}} \\
\hline
Method & F1-std ($\uparrow$) & F1-Area ($\uparrow$) & F1-Freq ($\uparrow$) & F1-Box ($\uparrow$) \\
\hline
BLT & 0.1673 & \second{0.3294} & 0.3716 & 0.7110 \\
LayoutDiff. & 0.1459 & 0.3095 & 0.3744 & \second{0.7716} \\
LayoutDM & 0.1496 & 0.2521 & \best{0.3951} & \best{0.7843} \\
\hline
DiffuseSG$_{\text{lay}}$ & \second{0.1753} & \second{0.3296} & \second{0.3939} & 0.7210 \\
{\name}$_{\text{lay}}$  & \best{0.1790} & \best{0.3302} & 0.3950 & 0.7540 \\
\hline
\multicolumn{5}{|c|}{\textbf{COCO}} \\
\hline
BLT & 0.4166 & 0.5056 & 0.6280 & 0.8000 \\
LayoutDiff. & 0.3586 & \second{0.5099} & 0.6290 & \best{0.8040} \\
LayoutDM & 0.3719 & 0.5095 & 0.6330 & \second{0.8110} \\
\hline
DiffuseSG$_{\text{lay}}$ & \second{0.4644} & \second{0.5153} & \second{0.6960} & 0.8023 \\
{\name}$_{\text{lay}}$  & \best{0.4724} & \best{0.5371} & \best{0.7130} & 0.8010 \\
\hline
\end{tabular}
\label{tab:layout_metrics}
\end{minipage}
\hfill
\begin{minipage}[t]{0.49\textwidth}
\centering
\caption{Comparing results of SG completion task (\colorbox{lightgreen}{Best}).}
\fontsize{7}{8}\selectfont
\begin{tabular}{|c|c|c|c|c|c|c|c|c|}
\hline
& \multicolumn{4}{c|}{Single Object $(\uparrow)$} & \multicolumn{4}{c|}{Single Relation $(\uparrow)$} \\
\hline
Baselines & w$_1$ & w$_{10}$ & w$_{50}$ & w$_{100}$ & w$_1$ & w$_{10}$ & w$_{50}$ & w$_{100}$ \\
\hline
\multicolumn{9}{|c|}{\textbf{Visual Genome}} \\
\hline
DiGress & 8.5 & 14.9 & 62.7 & 86.6 & 9 & 27.8 & 64.1 & 88.8 \\
DiffuseSG$_{\text{sg}}$ & 10.2 & 20.4 & 73.6 & 90 & 11.3 & 36.1 & 74.4 & 93.5 \\
{\name}$_{\text{sg}}$ & 16.7 & 28.1 & 82.9 & \best{95.3} & 19.4 & 43.6 & 87 & \best{98.7} \\
\hline
\multicolumn{9}{|c|}{\textbf{CompSGBench}} \\
\hline
DiGress & 5.9 & 10.9 & 55.5 & 82.5 & 7.7 & 19.8 & 50.1 & 70.5 \\
DiffuseSG$_{\text{sg}}$ & 8.6 & 19.3 & 66.9 & 85 & 8.8 & 27.9 & 70.3 & 83.5 \\
{\name}$_{\text{sg}}$  & 14.4 & 25.4 & 72.1 & \best{92.6} & 14 & 34.4 & 85.2 & \best{94.6} \\
\hline
\end{tabular}
\label{tab:sg_comple}
\end{minipage}
\vspace{-0.2in}
\end{table*}

\emph{Unconditional SG to Layout}:
Recall from Fig. \ref{fig:overview} that a layout head extracts layouts from the graph generator. 
Similar to graph evaluation, we sample $N=1000$ graphs and obtain their corresponding layouts. 
We restrict evaluation to scene graph datasets that have ground truth layout records. 
Table \ref{tab:layout_metrics} illustrates that {\name} achieves comparable performance to dedicated layout-baselines, namely LayoutDM and LayoutDiffusion. 
Of particular note is that {\name} beats DiffuseSG, despite the latter modeling layout as a diffusion state. 
The authors use the mismatched continuous diffusion framework to model scene graphs (and later encode them back to categorical results) for explicitly modeling layout as a continuous variable. 
However, these results indicate that by pursuing gains from layouts, DiffuseSG does not fully extract the gains from scene graphs. 

\emph{Scene graph completion}:
We report results on auxiliary scene graph tasks: single node and single relation completion. 
Following \cite{xu2024joint}, Table \ref{tab:sg_comple} reports win rates over $N=\{1,10,50,100\}$ graph samples, i.e., we mask one object or relation at random, and for such a masked sample, the model predicts $N$ completions. 
The win rate is computed as the number of graph samples out of $N$ that match the ground truth object and relation label. 
{\name} outperforms both DiffuseSG and DiGress. 
This is not surprising because {\name} is designed to explicitly learn object-relation semantics during the diffusion training. 
Figures \ref{fig:comple_vg}, \ref{fig:comple_vg_1}, \ref{fig:comple_comp}, \ref{fig:comple_comp_1} plot the corresponding qualitative completion results.

\circled{2} \textbf{Text-to-Image Generation}

Text-conditioned scene graphs, $G| T$, are sampled from the CLIP-reward tilted distribution (Eqn. \ref{eqn:reward-tilt}). 
The sampled graphs are passed as conditioning to an SG-Image generator, SDXL-SG. \cite{li2024laionsgenhancedlargescaledataset}.
The additional gain from SGs, over text-prompts alone, is evaluated by comparing {\name} against text-only image generation models: (1) SDXL, (2) ComposedDiffusion \cite{compose-diff}, an inference-time approach that breaks text prompts into concepts and guides image diffusion sampling with an algebraic score function composition of individual concepts, and (3) CO3 \cite{dutta2026steerawaymodecollisions}, also an inference-time intervention, that hypothesizes that image alignment collapses due to concept-mode collisions, and proposes a re-sampler/corrector algorithm to mitigate the issue.

Table \ref{tab:image_metrics_coco} demonstrates {\name}'s performance gains over text-only solutions, on popular composition metrics---ImageReward (IR) and BLIP-VQA. 
Figure \ref{fig:t2i_comp} shows some qualitative results, while Figure \ref{fig:t2i_sg} visualizes corresponding text-conditioned generated SGs.
Table \ref{tab:image_metrics_comp} extends Table \ref{tab:image_metrics_coco}'s  analysis to metrics measuring more granular structural attributes of the generated images for challenging scenarios (like long-prompts, multi-object and multi-relational spatial and semantic scenarios) in the CompSGBench benchmark. 
SG/E/R-IoU compare the structural components (scene graphs, objects, entities), detected in each method's final image, using an LLM. 
Gains are substantive across IoUs, ImageReward, and BLIP-VQA metrics, highlighting in essence that specialized text-based interventions are limited, and SGs---by constructing a more visual-scene compatible, structural representation of text--- offer robust conditioning signals for image synthesis. 
In particular, {\name} exhibits significant relation learning performance (Rel-IoU) without compromising object learning or structure, reinforcing the two key innovations: decoupling structure and semantics $E \to (E,R^+)$, and factorized sampler $V \to E \to R^+$.  
Additional qualitative results, including SGs for Figure \ref{fig:t2i_comp}, are in Appendix Sec. \ref{app:t2i_qual_res}.

Finally, we evaluate the grounding signals extracted by {\name}'s layout-head by using the GLIGEN backbone for layout-to-image generation \cite{li2023gligenopensetgroundedtexttoimage}. 
Baselines models, LayoutLLM-T2I and LDM, first use an LLM to translate text to layouts and then employ GLIGEN to perform layout-to-image generation. 
Table \ref{tab:image_metrics_layout} and Figure \ref{fig:t2i_layout} (qualitative)  reveal that {\name} can keep up with more powerful LLMs in extracting structural cues from input text. 
Thus, {\name} captures the inherent visual structure, generating visually consistent SGs.

\begin{table*}[tbp] 
\centering 
\caption{Comparing against text-only methods (\colorbox{lightgreen}{Best}, \colorbox{lightblue}{2nd best}).}
\small 
\begin{tabular}{@{}lccccc@{}} 
\toprule 
Method & FID ($\downarrow$) & CLIP-I2T ($\uparrow$) & CLIP-I2I ($\uparrow$) & BLIP-VQA ($\uparrow$) & ImageReward ($\uparrow$) \\ 
\midrule 
\multicolumn{6}{c}{\textbf{CompSGBench} \cite{li2023gligenopensetgroundedtexttoimage}} \\ 
\midrule 
SDXL      & \best{31.34} & \best{0.7415} & \best{0.3361} & 0.6425 & 1.1286 \\ 
ComposeDiff      & 49.72 & 0.6515 & 0.2844 & 0.3860 & -0.9213 \\ 
CO3       & 36.06 & 0.7348 & 0.3347 & \second{0.6724} & \second{1.2070} \\ 
{\name}$_{\text{sg}}$      & \second{33.85} & \second{0.7399} & \second{0.3350} & \best{0.7854} & \best{1.4512} \\ 
\midrule 
\multicolumn{6}{c}{\textbf{COCO} \cite{lin2015microsoftcococommonobjects}} \\ 
\midrule 
SDXL      & 43.02 & 0.7190 & 0.3177 & 0.6001 & 0.6409 \\ 
ComposeDiff      & 41.94 & 0.6518 & 0.2702 & 0.2981 & -1.2087 \\ 
CO3       & \second{40.26} & \second{0.7229} & \second{0.3218} & \second{0.6282} & \second{0.8857} \\ 
{\name}$_{\text{sg}}$     & \best{38.68} & \best{0.8228} & \best{0.3386} & \best{0.6734} & \best{1.1028} \\ 
\bottomrule 
\end{tabular} 
\label{tab:image_metrics_coco} 
\vspace{-0.07in}
\end{table*}


\begin{figure*}[h]
    \centering

    \begin{minipage}[t]{0.49\textwidth}
        \centering
        \includegraphics[width=\linewidth]{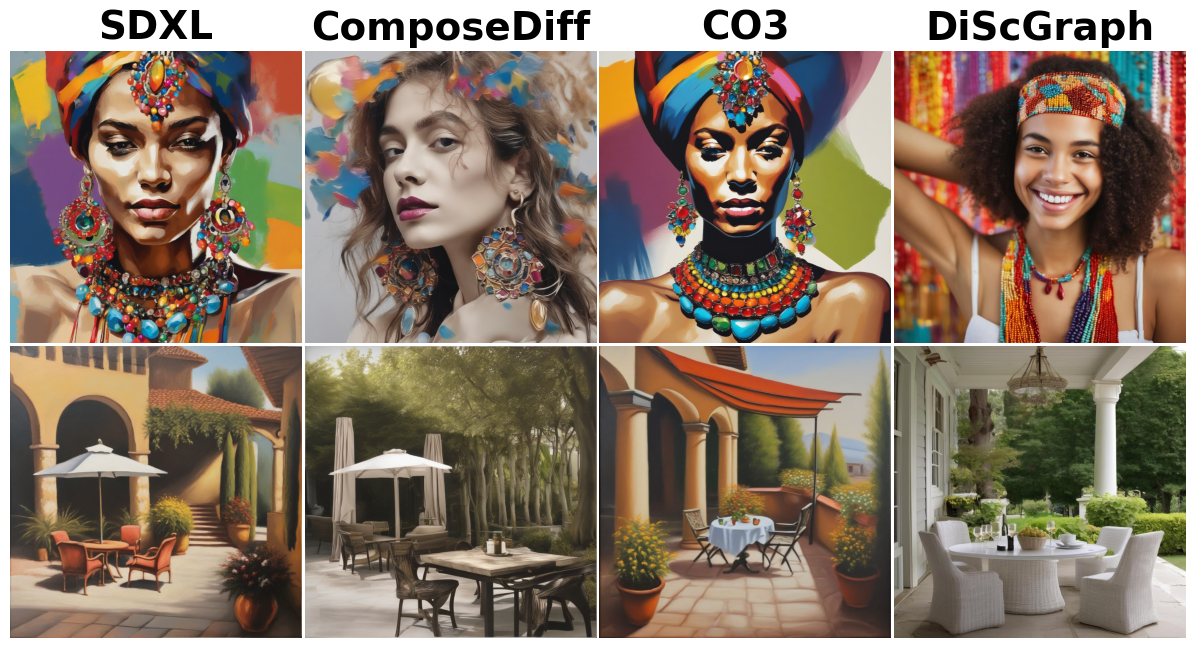}
        \includegraphics[width=\linewidth]{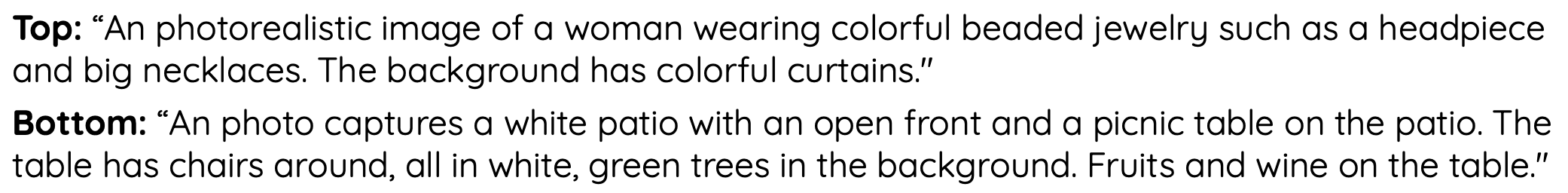}
    \end{minipage}
    \hfill
    \begin{minipage}[t]{0.49\textwidth}
        \centering
        \includegraphics[width=\linewidth]{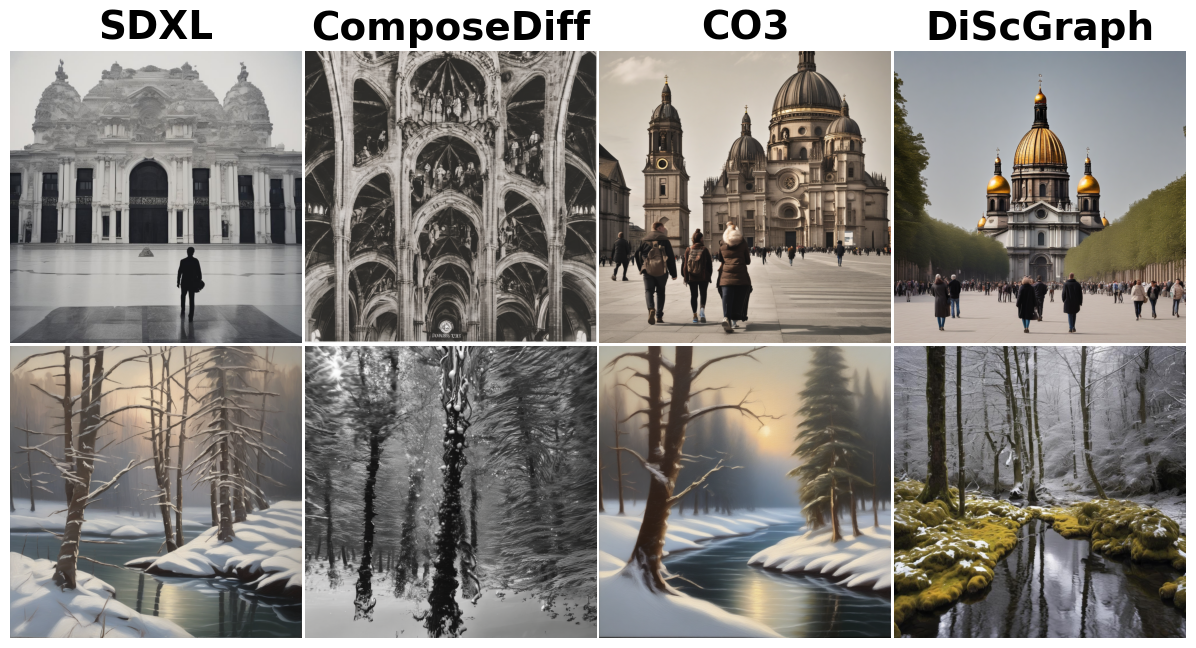}
        \includegraphics[width=\linewidth]{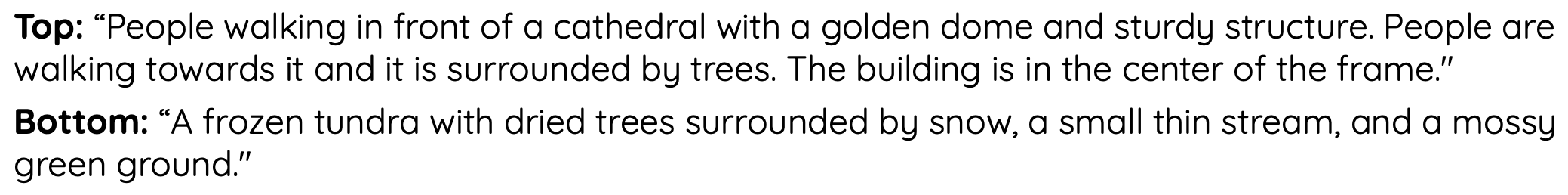}
    \end{minipage}
    \caption{\small{Qualitative results: Comparing T2I generation between SDXL, ComposeDiff, CO3, and {\name}.}}
    \label{fig:t2i_comp}
    \vspace{-0.8em}
\end{figure*}

\begin{figure}[hbt!]
        \centering
        \includegraphics[width=0.85\linewidth]{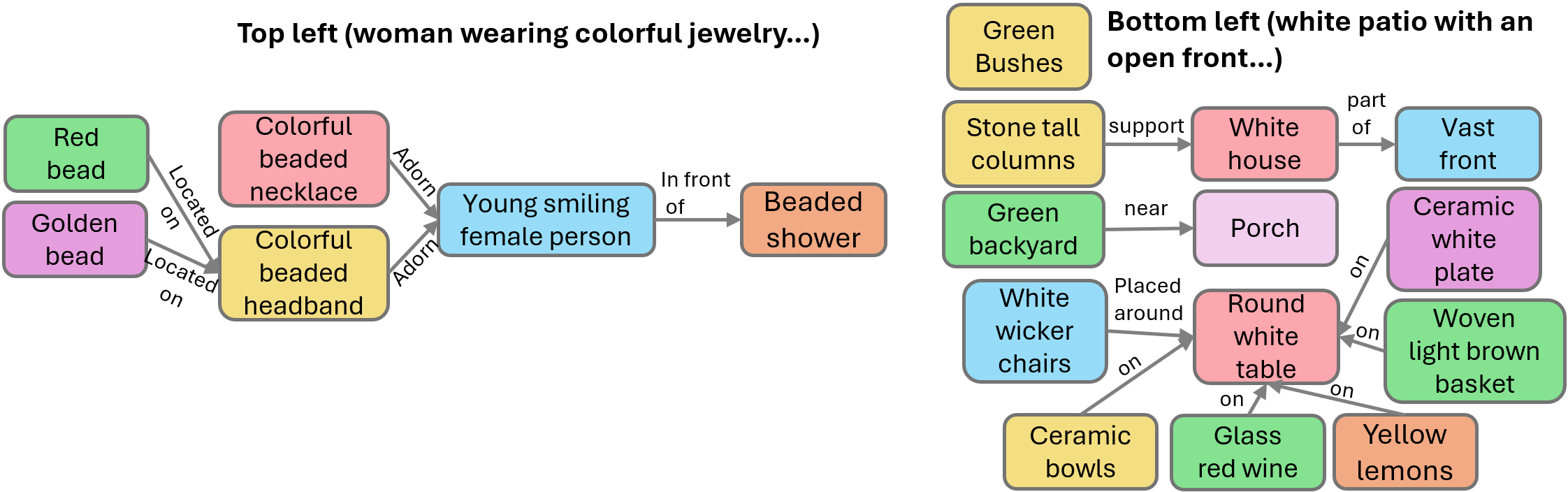}
    \caption{\small{Qualitative Results: Text-conditioned SGs used in SG-to-image results in Figure \ref{fig:t2i_comp}.}}
    \label{fig:t2i_sg}
    \vspace{-1.5em}
\end{figure}

\begin{table*}[tbp] 
\centering
\setlength{\tabcolsep}{3pt}
\renewcommand{\arraystretch}{1.2}

\begin{minipage}[t]{0.40\textwidth}
\centering
\caption{\small{Comparing text-only methods over structural image metrics (\colorbox{lightgreen}{Best}, \colorbox{lightblue}{2nd best}).}}
\fontsize{7}{8}\selectfont
\begin{tabular}{|c|c|c|c|}
\hline
\multicolumn{4}{|c|}{\textbf{CompSGBench}} \\
\hline
Method & SG-IoU ($\uparrow$) & E-IoU ($\uparrow$) & R-IoU ($\uparrow$) \\
\hline
SDXL & 0.3260 & 0.7620 & 0.7188 \\
ComposeDiff & 0.3250 & 0.7801 & 0.7169 \\
CO3 & \second{0.4271} & \second{0.7945} & \second{0.7203} \\
{\name}$_{\text{sg}}$ & \best{0.5629} & \best{0.8357} & \best{0.8299} \\
\hline
\end{tabular}
\label{tab:image_metrics_comp}
\end{minipage}
\hfill
\begin{minipage}[t]{0.59\textwidth}
\centering
\caption{\small{\textbf{Comparing layout head's results with layout-to-image methods (\colorbox{lightgreen}{Best}, \colorbox{lightblue}{2nd best}).}}}
\fontsize{7}{8}\selectfont
\begin{tabular}{|c|c|c|c|c|c|}
\hline
\multicolumn{6}{|c|}{\textbf{COCO}} \\
\hline
Method & FID ($\downarrow$) & CLIP-I2T ($\uparrow$) & CLIP-I2I ($\uparrow$) & BLIP-VQA ($\uparrow$) & IR ($\uparrow$) \\
\hline
LayoutDM & 72.32 & 0.5204 & 0.2955 & 0.6519 & 0.9206 \\
LayoutLLM-T2I & 69.01 & \second{0.6789} & \second{0.5373} & \best{0.8714} & 1.2966 \\
LDM & \best{68.84} & \best{0.6828} & 0.5386 & \second{0.8286} & \best{1.3552} \\
{\name}$_{\text{lay}}$ & \second{71.59} & 0.6802 & \best{0.5400} & 0.8599 & \second{1.3548} \\
\hline
\end{tabular}
\label{tab:image_metrics_layout}
\end{minipage}

\end{table*}

\circled{3} \textbf{Ablation Studies}

(1) \emph{Factorized Formulation}: 
We ablate on our main contributions: (1) the factorized state, and (2) the dependency-aware reverse sampler, $V\to E \to R$. 
As expected, a joint $E,R$ model for abstract graphs (using an independent $V,E$ sampler) achieves comparable performance to DiGress (Table \ref{tab:graph_metrics}). 
This joint model is then replaced by the factorized state $\mathcal{SG}:=\{(V,E,R^+), e_{ij} = 0 \implies r_{ij}=0\}$ and the corresponding hybrid corruption strategy is applied; however, the sampler is not dependency-aware and only respects the factorization i.e., $P(V_{t-1}\mid x_t)P(E_{t-1}\mid x_t)P(R_{t-1}\mid x_t,E_{t-1})$. 
The metrics---particularly over structure (ID/OD-MMD) and relation (R-MMD, RARE-K-TV)---improve substantially, validating the importance of $E,R^+$ factorization. 
Next, when the reverse sampler includes object-relation dependencies (See Eqn. \ref{eqn:rev_sampler}), the Triplet-TV metric improves, implying that the model captures semantic consistency between $V,E,R$ without compromising structure or relations. 
Finally, grounding cues can be extracted by adding the layout head,  without compromising the scene graph quality.

(2) We ablate over random and mask corruption ratios (Table \ref{tab:ablation_3}) and observe best performance when R/M=0.8. 
With pure random corruption (R/M=1), we see a drop in relation metrics (R-MMD, RARE-K-TV) and Triplet-TV, supporting our claims that absorbing-mask corruption boosts semantic learning by preserving structure and only deleting semantic information. 
On the other hand, with only mask corruption (R/M=0), the graph metrics are poor, indicating that the corruption does not meaningfully destroy the graph structure, impeding learning the graph distribution. 
When each corruption is applied equally (R/M=0.5), performance improves over R/M=1, but graph metrics are still poor. 
Thus, partial masking improves semantic learning, but it must only supplement random corruption, not replace it.
Appendix Sec. \ref{app:ablation} presents additional ablations.
\vspace{-0.05in}

\begin{table*}[tbp]
\centering

\setlength{\tabcolsep}{3pt}

\begin{minipage}[t]{0.49\textwidth}
\centering
\caption{Ablation on various components of DiScGraph: \textbf{1}=Joint, \textbf{2}=Factorized state, \textbf{3}=2+Dependency-aware sampler, \textbf{4}=3+Layout Head on VG (\colorbox{lightgreen}{Best}).}
\fontsize{6.5}{8.5}\selectfont
\begin{tabular}{|c|c|c|c|c|c|c|}
\hline
ID & N-MMD & E-MMD & ID-MMD & OD-MMD & TRIP-TV & R-K-TV\\
\hline
1 & 0.0077 & 0.0078 & 0.0068 & 0.0062 & 0.6910 & 0.8118\\
2 & 0.0066 & 0.0066 & 0.0056 & 0.0056 & 0.5822 & 0.6099\\
3 & 0.0046 & 0.0058 & 0.0043 & 0.0046 & 0.4878 & 0.5199\\
4 & \best{0.0046} & \best{0.0058} & \best{0.0043} & \best{0.0046} & \best{0.4920} & \best{0.5194}\\
\hline
\end{tabular}
\label{tab:ablation_0}
\end{minipage}
\hfill
\begin{minipage}[t]{0.49\textwidth}
\centering
\caption{\textbf{\small{Ablation on the different ratios of random and absorbing mask corruptions on Visual Genome (\colorbox{lightgreen}{Best}).\\}}}
\fontsize{6.5}{7.5}\selectfont
\begin{tabular}{|c|c|c|c|c|c|c|}
\hline
R/M & N-MMD & R-MMD & ID-MMD & OD-MMD & TRIP-TV & R-K-TV \\
\hline
1 & 0.0051 & 0.0118 & 0.0045 & 0.0049 & 0.6369 & 0.8326 \\
\best{0.8} & \best{0.0046} & \best{0.0058} & \best{0.0043} & \best{0.0046} & \best{0.4920} & \best{0.5194} \\
0.5 & 0.0058 & 0.0091 & 0.0063 & 0.0060 & 0.5735 & 0.7006 \\
0 & 0.4533 & 0.6099 & 0.9751 & 1.0000 & 0.9213 & 0.9992 \\
\hline
\end{tabular}
\label{tab:ablation_3}
\end{minipage}

\end{table*}

\section{Related Work}
\label{related}
\vspace{-0.1in}
$\blacksquare$ \textbf{Deep Graph Generative Models:} Pre-diffusion autoregressive, VAE, and GAN graph generators \cite{you2018graphrnn,liao2019gran,samanta2019nevae,li2018learning,ingraham2019generative,simonovsky2018graphvae,decao2018molgan} impose special constraints like node ordering, unlabelled, fixed-size graphs, poor scaling, and training instabilities. Concurrently, many works explored scene graph generation (VarScene, SceneGraphGen) \cite{verma2022varscene,garg2021unconditional} and prediction  \cite{newell2017pixels,li2018factorizable,xu2017scene,yang2018graph,zellers2018neural} tasks. \\
$\blacksquare$ \textbf{Diffusion Graph Generation.}
\textit{\textbf{(1) Continuous:}}
Diffusion models have replaced autoregressive approaches as the SOTA graph modeling paradigm. 
 Early graph diffusion models \cite{niu2020permutationinvariantgraphgeneration},\cite{jo2022score}, \cite{huang2022graphgdp} applied thresholding to snap continuous generations to categorical. DiffuseSG \cite{xu2024joint} adapts this idea for grounded scene graphs. Grounding augments standard scene graphs' node-object representation to include layouts. The rationale is layout is continuous while graph components are discrete, so grounded SGs are treated as continuous. 
\textit{\textbf{(2) Discrete diffusion models:}}
 Following the success of discrete diffusion for text, audio, point clouds etc. \cite{trippe2023diffusion,hoogeboom2022equivariant,wu2023diffmdgeometricdiffusionmodel}, DiGress \cite{vignac2023digress} designs a discrete diffusion model for abstract graphs (molecules, traffic, etc.). The forward process corrupts nodes and edges independently. The training enforces permutation invariance, and DiGress produces exchangeable graph distributions. While this works well for abstract graphs, scene graphs are inherently directed, have large vocabularies, long-tailed relation distributions, and expect semantic consistency between objects and relations. Our work {\name} addresses these scene graph-specific requirements with a hierarchical constrained factorized state space and corresponding factorized reverse sampler. \\
$\blacksquare$ \textbf{Controllable Image Generation:}
\textit{\textbf{(1) Layout \& LLM-augmented methods:}} Layout generation works~\citep{box-diff,kong2022blt,inoue2023layoutdm,zhou2023layoutdiffusion} use autoregressive and discrete diffusion paradigms to generate layouts or leverage LLM reasoning to infer layouts from text~\citep{layoutllm,lian2024llmgroundeddiffusionenhancingprompt}, With backbones like  GLIGEN \cite{li2023gligenopensetgroundedtexttoimage}, ControlNet \cite{zhang2023addingconditionalcontroltexttoimage}, they can generate layout-to-image. 
\textit{\textbf{(2) T2I Composition works:}} such as Composed-Diffusion~\citep{compose-diff} and CO3 \cite{dutta2026steerawaymodecollisions} enable algebraic composition of text concepts to rectify misaligned attributes.
\textit{\textbf{(3) SG-to-Image:}} Several works \cite{johnson2018image,yang2022diffusionbasedscenegraphimage,farshad2023scenegenie} train specific SG-to-image generators. SG-Adapter \cite{shen2024sgadapterenhancingtexttoimagegeneration} introduces SG IP-adapter that activates SG conditioning in existing image models. SDXL-SG \cite{li2024laionsgenhancedlargescaledataset} trains a large-scale SOTA image generator SDXL conditioned on scene graphs. SGEdit \cite{zhang2024sgedit} performs SG-to-image editing, and SGDiff \cite{zhang2025sgdiffscenegraphguided} generates SG to image captions.

\section{Limitations \& Future Scope}
\label{sec:limit}
\vspace{-0.1in}
We designed a scene graph generator as a factorized discrete diffusion model, {\name}, that has a factorized state space decoupling structure from semantics for better relation learning and a dependency-aware reverse sampler to sample SGs with strong entity dependence. Furthermore, we facilitated text-conditioned scene graph generation by sampling from a text-reward tilted distribution. The text-conditioned model was then used downstream in image generation. We provided empirical evidence on T2I generation improvements, unconditional scene graph generation metrics, and layout metrics. 

While {\name} achieves good scene graph generation and T2I generation performance, there is room for improvement,
(1) Our text conditioning framework operates at inference time through reward-tilted sampling, avoiding conditional model retraining. While this works well for textual abstractions -- SGs, the quality of conditioning depends heavily on the reward function (CLIP). Future work could explore stronger learned reward models. 
(2) Graph transformers scale quadratically with the number of scene entities and relations. While practical for current scene graph datasets, scaling to very large structured scenes may require techniques like sparse attention. 
(3) Discrete frameworks are restricted by sampler quality, requiring sampling refinements (Sec. \ref{app:sampler_refinements}, \ref{app:training}). For future work, we plan to address it in a more principled way by designing a robust sampling algorithm.   
(4) Layout prediction can become ambiguous for scene graphs with many disconnected sub-graphs. Future work could explore a latent layout generation that models multiple plausible global layouts conditioned on the same SG.
(5) Scene graph quality is dependent on the text input and is encumbered by unexpressive text prompts, limiting improvements to image generation. Limited availability of large-scale scene graph data and scene graph conditioned image generators also limits development; however, recent efforts like LAION-SG \cite{li2024laionsgenhancedlargescaledataset} are a step in the right direction.



\bibliography{reference}
\bibliographystyle{plain}

\newpage
\appendix
\section{Appendix}
\section*{Appendix Table of Contents} 
\etocsettocstyle{}{}             
\etocsetnexttocdepth{subsection} 
\localtableofcontents
\newpage

\subsection{Structure-aware Forward Process -- Additional Details}
In this section, we provide a detailed analysis of the forward corruption steady state convergence, marginals over object, edge, and relation entities, and different corruption strategies.
\subsubsection{Structure-Aware Forward Diffusion Convergence}
\label{app:fwd_conv}
We derive the convergence of the factorized forward process to a stationary distribution and show that it matches the scene graph data distribution formulation. 

We define a forward diffusion process over the structured scene graph state
\begin{equation} \begin{aligned}
x_t = (V_t, E_t, R_t^+),
\end{aligned} \end{equation} 
where object labels $V_t$, edge existence $E_t$, and relation labels $R_t^+$ satisfy the constraint
\begin{equation} \begin{aligned}
e_{ij,t} = 0 \;\Rightarrow\; r_{ij,t} = 0.
\end{aligned} \end{equation} 

The forward process is defined as a Markov chain:
\begin{equation} \begin{aligned}
q(x_{1:T} \mid x_0) = \prod_{t=1}^T q(x_t \mid x_{t-1}),
\end{aligned} \end{equation} 
with factorized transition kernel:
\begin{equation} \begin{aligned}
q(x_t \mid x_{t-1})
=
q_V(V_t \mid V_{t-1})
q_E(E_t \mid E_{t-1})
q_R(R_t \mid R_{t-1}, E_t).
\end{aligned} \end{equation} 

\paragraph{Categorical Corruption Kernel.}
For a categorical variable $z_t \in \{1,\dots,K\}$, we define
\begin{equation} \begin{aligned}
Q_t = (1 - \beta_t) I + \beta_t \mathbf{1}\pi^\top,
\end{aligned} \end{equation} 
where $\pi \in \Delta^{K-1}$ is a fixed prior distribution. The transition is
\begin{equation} \begin{aligned}
q(z_t = b \mid z_{t-1} = a) = [Q_t]_{ab}.
\end{aligned} \end{equation} 

In the time-homogeneous case,
\begin{equation} \begin{aligned}
Q = (1-\beta)I + \beta \mathbf{1}\pi^\top,
\end{aligned} \end{equation} 
we have 
\begin{equation} \begin{aligned}
\pi^\top Q = \pi^\top,
\end{aligned} \end{equation} 
so $\pi$ is a stationary distribution.

Moreover,
\begin{equation} \begin{aligned}
p_t^\top - \pi^\top = (1-\beta)^t (p_0^\top - \pi^\top),
\end{aligned} \end{equation} 
so
\begin{equation} \begin{aligned}
p_t \to \pi \quad \text{as } t \to \infty.
\end{aligned} \end{equation} 

In the time non-homogeneous case,
\begin{equation} \begin{aligned}
Q_t = (1-\beta_t)I + \beta_t \mathbf{1}\pi^\top,
\end{aligned} \end{equation} 
we obtain
\begin{equation} \begin{aligned}
p_t^\top - \pi^\top = \left(\prod_{s=1}^t (1-\beta_s)\right)(p_0^\top - \pi^\top),
\end{aligned} \end{equation} 
and convergence holds if
\begin{equation} \begin{aligned}
\prod_{s=1}^t (1-\beta_s) \to 0.
\end{aligned} \end{equation} 

\paragraph{Object and Edge Corruption.}
Object labels and edge existence follow standard discrete diffusion corruption,
\begin{equation}
q_V(v_{i,t} \mid v_{i,t-1})
=
(1-\beta_t^V)\delta_{v_{i,t-1}}
+
\beta_t^V \pi_V,
\end{equation}
\begin{equation} \begin{aligned}
q_V(V_t \mid V_{t-1}) = \prod_i q_V(v_{i,t} \mid v_{i,t-1}),
\end{aligned} \end{equation} 
\begin{equation}
Q_t^V = (1-\beta_t^V)I + \beta_t^V \mathbf{1}\pi_V^\top.
\end{equation}
where $\pi_V$ is a prior distribution over object labels, can be chosen as uniform or the empirical training data distribution, and $\delta$ is a point-mass distribution.

Each directed edge is, $e_{ij,t} \in \{0,1\}.$
\begin{equation}
q_E(e_{ij,t} \mid e_{ij,t-1})
=
(1-\beta_t^E)\delta_{e_{ij,t-1}}
+
\beta_t^E \pi_E,
\end{equation}
with  prior from training data statistics on edge existence,  $\pi_E(1)=\rho_E, \quad \pi_E(0)=1-\rho_E.$
\begin{equation} \begin{aligned}
q_E(E_t \mid E_{t-1}) = \prod_{i \neq j} q_E(e_{ij,t} \mid e_{ij,t-1}).
\end{aligned} \end{equation} 
\begin{equation}
Q_t^E = (1-\beta_t^E)I + \beta_t^E \mathbf{1}\pi_E^\top.
\end{equation}
Under the above conditions on $\beta_t^V$ and $\beta_t^E$, we obtain
\begin{equation} \begin{aligned}
q(V_t) \to \pi_V(V) = \prod_i \pi_V(v_i),
\end{aligned} \end{equation} 
\begin{equation} \begin{aligned}
q(E_t) \to \pi_E(E) = \prod_{i \neq j} \pi_E(e_{ij}).
\end{aligned} \end{equation} 

\paragraph{Structure-Aware Relation Corruption.}
Relation corruption is conditioned on edge existence,
\begin{equation} \begin{aligned}
q_R(r_{ij,t} \mid r_{ij,t-1}, e_{ij,t}) =
\begin{cases}
\Tilde q_R(r_{ij,t} \mid r_{ij,t-1}), & e_{ij,t} = 1,\\
\delta_0(r_{ij,t}), & e_{ij,t} = 0,
\end{cases}
\end{aligned} \end{equation} 
where $\delta_0$ is the point mass at $0$, and
\begin{equation} \begin{aligned}
\Tilde Q_t^R = (1-\beta_t^R)I + \beta_t^R \mathbf{1}\pi_R^\top.
\end{aligned} \end{equation} 

Thus, for active edges,
\begin{equation} \begin{aligned}
\Tilde q_R(r_{ij,t} \mid e_{ij,t} = 1) \to \pi_R(r_{ij}),
\end{aligned} \end{equation} 
while for inactive edges,
\begin{equation} \begin{aligned}
q(r_{ij,t} = 0 \mid e_{ij,t} = 0) = 1.
\end{aligned} \end{equation} 

\paragraph{Structured Stationary Distribution.}
The limiting distribution over scene graphs is therefore
\begin{equation} \begin{aligned}
\pi_{\mathrm{SG}}(V,E,R)
=
\prod_i \pi_V(v_i)
\prod_{i \neq j} \pi_E(e_{ij})
\prod_{i \neq j : e_{ij}=1} \pi_R(r_{ij})
\prod_{i \neq j : e_{ij}=0} \delta_0(r_{ij}).
\end{aligned} \end{equation} 

Equivalently,
\begin{equation} \begin{aligned}
\pi_{\mathrm{SG}}(V,E,R)
=
\pi_V(V)\,\pi_E(E)\,\pi_R(R \mid E),
\end{aligned} \end{equation} 
where
\begin{equation} \begin{aligned}
\pi_R(R \mid E)
=
\prod_{e_{ij}=1} \pi_R(r_{ij})
\prod_{e_{ij}=0} \delta_0(r_{ij}).
\end{aligned} \end{equation} 

This distribution is supported only on valid scene graphs.

\begin{lemma}
Factorized diffusion state convergence:
Assume that
\begin{equation} \begin{aligned}
\prod_{s=1}^t (1-\beta_s^V) \to 0,\quad
\prod_{s=1}^t (1-\beta_s^E) \to 0,\quad
\prod_{s=1}^t (1-\beta_s^R) \to 0.
\end{aligned} \end{equation} 
Then the forward process satisfies
\begin{equation} \begin{aligned}
q(x_t \mid x_0) \to \pi_{\mathrm{SG}}(V,E,R).
\end{aligned} \end{equation} 
\end{lemma}
\textbf{Proof.}  
From categorical diffusion convergence,
\begin{equation} \begin{aligned}
q(V_t \mid V_0) \to \pi_V(V), \quad
q(E_t \mid E_0) \to \pi_E(E).
\end{aligned} \end{equation} 

For each pair $(i,j)$, conditional on $e_{ij,t}$,
\begin{equation} \begin{aligned}
q(r_{ij,t} \mid r_{ij,0}, e_{ij,t})
\to
\begin{cases}
\pi_R(r_{ij}), & e_{ij,t} = 1,\\
\delta_0(r_{ij}), & e_{ij,t} = 0.
\end{cases}
\end{aligned} \end{equation} 

Thus,
\begin{equation} \begin{aligned}
q(R_t \mid R_0, E_t)
\to
\prod_{i \neq j : e_{ij,t}=1} \pi_R(r_{ij,t})
\prod_{i \neq j : e_{ij,t}=0} \delta_0(r_{ij,t}).
\end{aligned} \end{equation} 

Combining all components,
\begin{equation} \begin{aligned}
q(x_t \mid x_0)
=
q(V_t \mid V_0)
q(E_t \mid E_0)
q(R_t \mid R_0, E_t)
\to
\pi_{\mathrm{SG}}(V,E,R).
\end{aligned} \end{equation} 
The data distribution can be factorized as
\begin{equation}
p_{\text{data}}(G) = p_{\text{data}}(V)\, p_{\text{data}}(E \mid V)\, p_{\text{data}}(R^+ \mid E,V),
\end{equation}
which exposes the hierarchical visual scene structure:
objects $\rightarrow$ edges $\rightarrow$ relations.

Additionally, 

\textbf{Proposition.} If $x_{t-1} \in \mathcal{G}_{SG}$, then $x_t \sim q(x_t \mid x_{t-1})$ also lies in $\mathcal{G}_{SG}$.

From the forward corruption formulation, 
\begin{equation}
\Tilde{q}_R(r_{ij,t} \mid r_{ij,t-1},e_{ij,t}=0) = \delta_0(r_ij,t)
\end{equation}
Then the only non-zero probability relation is $r_{ij,t} = 0$.
Thus, If $x_{t-1} \in \mathcal{G}_{SG}$, then $x_t \sim q(x_t \mid x_{t-1})$ also lies in $\mathcal{G}_{SG}$, thus the valid scene graph factorized space is preserved at all timesteps in the forward process.

\subsubsection{Closed-Form Marginals and Corruption Variants}
\label{app:fwd_marginal}
A key property of discrete diffusion is the ability to directly sample from the marginal distribution $q(x_t \mid x_0)$ via closed-form transition matrices. We now state these marginals for the factorized scene-graph state.

\paragraph{Closed-form marginal for objects.}
Following discrete diffusion formulation \cite{austin2023structureddenoisingdiffusionmodels}, let $\{Q_s^V\}_{s=1}^t$ denote the object transition matrices and define the cumulative transition
\begin{equation}
\bar{Q}_t^V = Q_1^V Q_2^V \cdots Q_t^V.
\end{equation}
Then the marginal distribution at time $t$ is
\begin{equation}
q(v_{i,t} \mid v_{i,0}) = \mathrm{Cat}\!\left(v_{i,0}\, \bar{Q}_t^V\right).
\end{equation}
If $v_{i,0}$ is represented as a one-hot vector, then
\begin{equation}
p(v_{i,t} = c \mid v_{i,0}) = \left[v_{i,0}\, \bar{Q}_t^V\right]_c.
\end{equation}

\paragraph{Closed-form marginal for edges.}
Similarly, let $\bar{Q}_t^E = Q_1^E Q_2^E \cdots Q_t^E$. Then
\begin{equation}
q(e_{ij,t} \mid e_{ij,0}) = \mathrm{Cat}\!\left(e_{ij,0}\, \bar{Q}_t^E\right).
\end{equation}
For binary edges,
\begin{equation}
p(e_{ij,t} = b \mid e_{ij,0}) = \left[e_{ij,0}\, \bar{Q}_t^E\right]_b.
\end{equation}

\paragraph{Closed-form marginal for relations.} The relation marginals need to be reworked for the factorized scene graph space $\mathcal{SG}$. Relation marginals depend on the edge state at time $t$. Conditioned on an active edge,
\begin{equation}
q(r_{ij,t} \mid r_{ij,0}, e_{ij,t}=1)
= \mathrm{Cat}\!\left(r_{ij,0}\, \bar{Q}_t^R\right),
\end{equation}
where $\bar{Q}_t^R = Q_1^R Q_2^R \cdots Q_t^R$.

If the edge is inactive,
\begin{equation}
q(r_{ij,t} = 0 \mid r_{ij,0}, e_{ij,t}=0) = 1.
\end{equation}
Thus, the structural components, edges, are decoupled from the semantics, relations, but the dependency structure ensures we are operating in valid scene graph space. 
\begin{equation}
q(r_{ij,t} \mid r_{ij,0}, e_{ij,t}) =
\begin{cases}
\mathrm{Cat}\!\left(r_{ij,0}\, \bar{Q}_t^R\right), & e_{ij,t}=1, \\
\delta_0(r_{ij,t}), & e_{ij,t}=0.
\end{cases}
\end{equation}

\paragraph{Full closed-form marginal.}
Combining all components,
\begin{equation}
q(x_t \mid x_0)
=
q_V(V_t \mid V_0)\,
q_E(E_t \mid E_0)\,
q_R(R_t \mid R_0, E_t).
\end{equation}
\begin{equation}
    \begin{aligned}
q(x_t \mid x_0)
&=
\prod_i q_V(v_{i,t} \mid v_{i,0})
\prod_{i\neq j} q_E(e_{ij,t} \mid e_{ij,0}) \\
&\quad \cdot
\prod_{ij:e_{ij,t}=1} q_R(r_{ij,t} \mid r_{ij,0}, e_{ij,t}=1)
\prod_{ij:e_{ij,t}=0} \delta_0(r_{ij,t}).
\end{aligned}
\end{equation}

This enables direct sampling of noisy states of the factorized space during training.

\subsubsection{Forward Corruption Choices}
\label{app:fwd_mask}
\paragraph{Absorbing-mask corruption.}
 Given the success of absorbing masks for diffusion language tasks, apart from random corruption, we apply absorbing mask corruption, which preserves the structure and only deletes information (object and relation), allowing the model to specifically learn semantics, which is beneficial for scene graphs.
We define mask tokens $[\mathrm{MASK}]_V$ and $[\mathrm{MASK}]_R$.

For objects:
\begin{equation}
q_V(v_{i,t} \mid v_{i,t-1})
=
(1-\beta_t^V)\delta_{v_{i,t-1}}
+
\beta_t^V \delta_{[\mathrm{MASK}]_V}.
\end{equation}

For relations on active edges:
\begin{equation}
\Tilde{q}_R(r_{ij,t} \mid r_{ij,t-1})
=
(1-\beta_t^R)\delta_{r_{ij,t-1}}
+
\beta_t^R \delta_{[\mathrm{MASK}]_R}.
\end{equation}

The edge-aware kernel becomes
\begin{equation}
q_R(r_{ij,t} \mid r_{ij,t-1}, e_{ij,t}) =
\begin{cases}
(1-\beta_t^R)\delta_{r_{ij,t-1}} + \beta_t^R \delta_{[\mathrm{MASK}]_R}, & e_{ij,t}=1, \\
\delta_0(r_{ij,t}), & e_{ij,t}=0.
\end{cases}
\end{equation}
This restricts masking to valid relations.

\vspace{0.5em}
\paragraph{Hybrid random + mask corruption.}
Performing masking in isolation would not allow the model to learn structure; we need a hybrid kernel combining random replacement and mask absorption, thereby learning both structure and semantics.

For objects:
\begin{equation}
q_V(v_{i,t} \mid v_{i,t-1})
=
(1-\beta_t^V)\delta_{v_{i,t-1}}
+
\beta_t^V\big[
\rho_V \delta_{[\mathrm{MASK}]_V}
+
(1-\rho_V)\pi_V
\big].
\end{equation}

For active relations:
\begin{equation}
\Tilde{q}_R(r_{ij,t} \mid r_{ij,t-1})
=
(1-\beta_t^R)\delta_{r_{ij,t-1}}
+
\beta_t^R\big[
\rho_R \delta_{[\mathrm{MASK}]_R}
+
(1-\rho_R)\pi_R
\big].
\end{equation}

The edge-aware relation kernel remains
\begin{equation}
q_R(r_{ij,t} \mid r_{ij,t-1}, e_{ij,t}) =
\begin{cases}
\Tilde{q}_R(r_{ij,t} \mid r_{ij,t-1}), & e_{ij,t}=1, \\
\delta_0(r_{ij,t}), & e_{ij,t}=0.
\end{cases}
\end{equation}

Thus, we defined a structure-aware forward process over $x_t = (V_t, E_t, R_t^+)$
such that $q(x_t \mid x_{t-1}) = q_V q_E q_R,$
with,
\begin{equation}
q_R(R_t \mid R_{t-1}, E_t)
=
\prod_{ij:e_{ij,t}=1} \Tilde{q}_R(r_{ij,t} \mid r_{ij,t-1})
\prod_{ij:e_{ij,t}=0} \delta_0(r_{ij,t}).
\end{equation}

This construction:
(i) preserves validity of the scene-graph state space,
(ii) decouples structural sparsity from relation semantics, and
(iii) prevents rare relations from competing with the no-edge class.

\subsection{Factorized Reverse Sampler -- Additional Details}
In this section, we discuss posterior derivation for the factorized scene graph model and refinements to the sampler at inference.
\subsubsection{Factorized Reverse Model and Structured Posterior Sampler}
\label{app:rev_sample}

\paragraph{Discrete posterior for a categorical variable.} \cite{austin2023structureddenoisingdiffusionmodels,vignac2023digress}
Consider a categorical variable \(z_t\in\{1,\dots,K\}\) with forward transition matrices
\begin{equation} \begin{aligned}
Q_t,\qquad \bar Q_t = Q_1Q_2\cdots Q_t.
\end{aligned} \end{equation} 
The one-step posterior conditioned on the clean variable \(z_0\) is
\begin{equation} \begin{aligned}
q(z_{t-1}=a\mid z_t=b,z_0=c)
=
\frac{
q(z_t=b\mid z_{t-1}=a)\,
q(z_{t-1}=a\mid z_0=c)
}{
q(z_t=b\mid z_0=c)
}.
\end{aligned} \end{equation} 
Using the transition matrices, this becomes
\begin{equation} \begin{aligned}
q(z_{t-1}=a\mid z_t=b,z_0=c)
=
\frac{
[Q_t]_{ab}\,[\bar Q_{t-1}]_{ca}
}{
[\bar Q_t]_{cb}
}.
\end{aligned} \end{equation} 

Since \(z_0\) is unknown at sampling time, the model predicts a distribution
\begin{equation} \begin{aligned}
p_\theta(z_0=c\mid x_t).
\end{aligned} \end{equation} 
The learned reverse posterior marginalizes over \(z_0\):
\begin{equation} \begin{aligned}
p_\theta(z_{t-1}=a\mid z_t=b)
=
\sum_c
q(z_{t-1}=a\mid z_t=b,z_0=c)\,
p_\theta(z_0=c\mid x_t).
\end{aligned} \end{equation} 
Equivalently,
\begin{equation} \begin{aligned}
p_\theta(z_{t-1}=a\mid z_t=b)
\propto
q(z_t=b\mid z_{t-1}=a)
\sum_c
p_\theta(z_0=c\mid x_t)
q(z_{t-1}=a\mid z_0=c).
\end{aligned} \end{equation} 

\paragraph{Reverse factorization for scene graphs.}
For scene graphs, the clean state is
\begin{equation} \begin{aligned}
x_0=(V_0,E_0,R_0).
\end{aligned} \end{equation} 
Following the structured state space, we factorize the clean-state prediction as
\begin{equation} \begin{aligned}
p_\theta(x_0\mid x_t)
=
p_\theta(V_0,E_0,R_0\mid x_t).
\end{aligned} \end{equation} 

A scene-graph-aware factorization is
\begin{equation} \begin{aligned}
p_\theta(V_0,E_0,R_0\mid x_t)
=
p_\theta(V_0\mid x_t)\,
p_\theta(E_0\mid V_0,x_t)\,
p_\theta(R_0\mid V_0,E_0,x_t).
\end{aligned} \end{equation} 

This factorization follows the semantic structure of scene graphs:
\begin{enumerate}
    \item object identities define the semantic entities in the graph,
    \item directed edge existence determines which object pairs interact,
    \item relation identities are meaningful only for active directed edges.
\end{enumerate}

Thus, relation prediction is conditioned on both object identity and edge existence.

\paragraph{Object reverse posterior.}
For each object variable \(v_i\), the model predicts, $p_\theta(v_{i,0}=c\mid x_t).$
The reverse posterior is
\begin{equation} \begin{aligned}
p_\theta(v_{i,t-1}=a\mid v_{i,t}=b,x_t)
=
\sum_c
q_V(v_{i,t-1}=a\mid v_{i,t}=b,v_{i,0}=c)\,
p_\theta(v_{i,0}=c\mid x_t).
\end{aligned} \end{equation} 
Using the forward kernels,
\begin{equation} \begin{aligned}
q_V(v_{i,t-1}=a\mid v_{i,t}=b,v_{i,0}=c)
=
\frac{
[Q_t^V]_{ab}\,[\bar Q_{t-1}^V]_{ca}
}{
[\bar Q_t^V]_{cb}
}.
\end{aligned} \end{equation} 
Therefore,
\begin{equation} \begin{aligned}
p_\theta(v_{i,t-1}=a\mid v_{i,t}=b,x_t)
=
\sum_c
\frac{
[Q_t^V]_{ab}\,[\bar Q_{t-1}^V]_{ca}
}{
[\bar Q_t^V]_{cb}
}
p_\theta(v_{i,0}=c\mid x_t).
\end{aligned} \end{equation} 

\paragraph{Edge reverse posterior.}
For each directed edge variable \(e_{ij}\in\{0,1\}\), the model predicts, $p_\theta(e_{ij,0}=c\mid V_0,x_t).$
The reverse posterior is
\begin{equation} \begin{aligned}
p_\theta(e_{ij,t-1}=a\mid e_{ij,t}=b,V_0,x_t)
=
\sum_c
q_E(e_{ij,t-1}=a\mid e_{ij,t}=b,e_{ij,0}=c)
p_\theta(e_{ij,0}=c\mid V_0,x_t),
\end{aligned} \end{equation} 
where
\begin{equation} \begin{aligned}
q_E(e_{ij,t-1}=a\mid e_{ij,t}=b,e_{ij,0}=c)
=
\frac{
[Q_t^E]_{ab}\,[\bar Q_{t-1}^E]_{ca}
}{
[\bar Q_t^E]_{cb}
}.
\end{aligned} \end{equation} 

Thus,
\begin{equation} \begin{aligned}
p_\theta(e_{ij,t-1}=a\mid e_{ij,t}=b,V_0,x_t)
=
\sum_c
\frac{
[Q_t^E]_{ab}\,[\bar Q_{t-1}^E]_{ca}
}{
[\bar Q_t^E]_{cb}
}
p_\theta(e_{ij,0}=c\mid V_0,x_t).
\end{aligned} \end{equation}

\paragraph{Edge-gated relation reverse posterior.}
Relations are only meaningful on active directed edges. Since the forward
relation corruption is conditioned on the current edge state \(e_{ij,t}\), the
reverse relation posterior must distinguish three cases depending on
\((e_{ij,t-1}, e_{ij,t})\).

Let the denoiser output clean relation logits after edge and object predictions,
\begin{equation} \begin{aligned}
\epsilon_\theta^R(x_t,t,V_{t-1},E_{t-1})_{ij}
\in \mathbb{R}^{K_{\mathrm{rel}}},
\end{aligned} \end{equation}
and the predicted clean relation distribution on active edges is,
\begin{equation} \begin{aligned}
p_\theta(r_{ij,0}=c \mid x_t,V_{t-1},E_{t-1})
=
\mathrm{softmax}
\left(
\epsilon_\theta^R(x_t,t,V_{t-1},E_{t-1})_{ij}
\right)_c.
\end{aligned} \end{equation}

\noindent\textbf{Case 1: Edge absent at \(t-1\).}
If $e_{ij,t-1}=0,$
then the relation is deterministically null:
\begin{equation} \begin{aligned}
p_\theta(r_{ij,t-1}=0\mid e_{ij,t-1}=0,x_t)=1.
\end{aligned} \end{equation}

\noindent\textbf{Case 2: Edge active at both \(t-1\) and \(t\).}
If $e_{ij,t-1}=1,\quad e_{ij,t}=1,$
then relation corruption was active in the forward process. Therefore we can use
the standard discrete posterior over positive relation labels:
\begin{equation} \begin{aligned}
p_\theta(r_{ij,t-1}=a
\mid r_{ij,t}=b,e_{ij,t-1}=1,e_{ij,t}=1,x_t)
&=
\sum_c
q_R(r_{ij,t-1}=a\mid r_{ij,t}=b,r_{ij,0}=c)\\
&p_\theta(r_{ij,0}=c\mid x_t,V_{t-1},E_{t-1}),
\end{aligned} \end{equation}
where
\begin{equation} \begin{aligned}
q_R(r_{ij,t-1}=a\mid r_{ij,t}=b,r_{ij,0}=c)
=
\frac{
[Q_t^R]_{ab}\,[\bar Q_{t-1}^R]_{ca}
}{
[\bar Q_t^R]_{cb}
}.
\end{aligned} \end{equation}
Equivalently,
\begin{equation} \begin{aligned}
p_\theta(r_{ij,t-1}=a
\mid r_{ij,t}=b,e_{ij,t-1}=1,e_{ij,t}=1,x_t)
&\propto
q_R(r_{ij,t}=b\mid r_{ij,t-1}=a)\\
&\sum_c
p_\theta(r_{ij,0}=c\mid x_t,V_{t-1},E_{t-1})
q_R(r_{ij,t-1}=a\mid r_{ij,0}=c).
\end{aligned} \end{equation}

\noindent\textbf{Case 3: Edge reactivates in the reverse step.}
The key sampler modification for the factorized setting is in addressing this case. If $e_{ij,t-1}=1,\quad e_{ij,t}=0,$
then the current relation \(r_{ij,t}=0\) is caused deterministically by the
absence of the current edge. Hence \(r_{ij,t}\) carries no semantic information
about the positive relation at \(t-1\). We therefore use the predictive marginal
obtained by propagating the denoiser's clean relation prediction to timestep
\(t-1\):
\begin{equation} \begin{aligned}
p_\theta(r_{ij,t-1}=a
\mid e_{ij,t-1}=1,e_{ij,t}=0,x_t)
=
\sum_c
q_R(r_{ij,t-1}=a\mid r_{ij,0}=c)
p_\theta(r_{ij,0}=c\mid x_t,V_{t-1},E_{t-1}).
\end{aligned} \end{equation}
Using the relation marginal transition matrix,
\begin{equation} \begin{aligned}
q_R(r_{ij,t-1}=a\mid r_{ij,0}=c)
=
[\bar Q_{t-1}^R]_{ca},
\end{aligned} \end{equation}
so
\begin{equation} \begin{aligned}
p_\theta(r_{ij,t-1}=a
\mid e_{ij,t-1}=1,e_{ij,t}=0,x_t)
=
\sum_c
[\bar Q_{t-1}^R]_{ca}
p_\theta(r_{ij,0}=c\mid x_t,V_{t-1},E_{t-1}).
\end{aligned} \end{equation}

Thus the full edge-gated relation sampler is
\begin{equation} \begin{aligned}
p_\theta(r_{ij,t-1}\mid r_{ij,t},e_{ij,t-1},e_{ij,t},x_t)
=
\begin{cases}
\delta_0(r_{ij,t-1}),
& e_{ij,t-1}=0,\\[6pt]
p_\theta^{\mathrm{post}}(r_{ij,t-1}\mid r_{ij,t},x_t),
& e_{ij,t-1}=1,\ e_{ij,t}=1,\\[6pt]
p_\theta^{\mathrm{marg}}(r_{ij,t-1}\mid x_t),
& e_{ij,t-1}=1,\ e_{ij,t}=0.
\end{cases}
\end{aligned} \label{eqn:three_cases}\end{equation}

\paragraph{Structured reverse transition.}
Combining the object, edge, and relation posterior updates gives the structured reverse transition,
\begin{equation} \begin{aligned}
p_\theta(x_{t-1}\mid x_t)
=
p_\theta(V_{t-1}\mid x_t)\,
p_\theta(E_{t-1}\mid V_{t-1},x_t)\,
p_\theta(R_{t-1}\mid V_{t-1},E_{t-1},x_t).
\end{aligned} \end{equation} 

\begin{equation} \begin{aligned}
p_\theta(x_{t-1}\mid x_t)
=
\prod_i p_\theta(v_{i,t-1}\mid v_{i,t},x_t)
\prod_{i\neq j}p_\theta(e_{ij,t-1}\mid e_{ij,t},V_{t-1},x_t)
\prod_{i\neq j}p_\theta(r_{ij,t-1}\mid r_{ij,t},e_{ij,t-1},V_{t-1},E_{t-1},x_t).
\end{aligned} \label{eqn:posterior} \end{equation} 

The final relation factor in Eqn. \ref{eqn:posterior} is edge-gated (Eqn. \ref{eqn:three_cases}).

\textbf{Proposition.}
Validity preservation of the reverse sampler:
  
If the reverse sampler uses the edge-gated relation posterior above, then every sampled state \(x_{t-1}\) satisfies
\begin{equation} \begin{aligned}
e_{ij,t-1}=0 \Rightarrow r_{ij,t-1}=0.
\end{aligned} \end{equation} 

For any directed pair \((i,j)\), suppose $e_{ij,t-1}=0.$
By definition of the edge-gated relation posterior,
\begin{equation} \begin{aligned}
p_\theta(r_{ij,t-1}=0\mid e_{ij,t-1}=0,x_t)=1.
\end{aligned} \end{equation} 
Therefore the only relation value with nonzero probability is \(r_{ij,t-1}=0\). Hence
\begin{equation} \begin{aligned}
e_{ij,t-1}=0 \Rightarrow r_{ij,t-1}=0.
\end{aligned} \end{equation} 
Thus the sampled state remains in the constrained scene graph space.

\paragraph{Sampling Process.}
At each reverse timestep \(t\), we perform:
\begin{enumerate}
    \item Predict object clean-state probabilities:
    \begin{equation} \begin{aligned}
    p_\theta(V_0\mid x_t).
    \end{aligned} \end{equation} 
    Sample or decode
    \begin{equation} \begin{aligned}
    V_{t-1}\sim p_\theta(V_{t-1}\mid x_t)
    \end{aligned} \end{equation} 
    using the object posterior.

    \item Predict edge clean-state probabilities conditioned on recovered objects:
    \begin{equation} \begin{aligned}
    p_\theta(E_0\mid V_{t-1},x_t).
    \end{aligned} \end{equation} 
    Sample
    \begin{equation} \begin{aligned}
    E_{t-1}\sim p_\theta(E_{t-1}\mid V_{t-1},x_t)
    \end{aligned} \end{equation} 
    using the edge posterior.

    \item Predict relation clean-state probabilities conditioned on recovered objects and edges:
    \begin{equation} \begin{aligned}
    p_\theta(R_0\mid V_{t-1},E_{t-1},x_t).
    \end{aligned} \end{equation} 
    For each pair \((i,j)\):
    \begin{equation} \begin{aligned}
    r_{ij,t-1}=
    \begin{cases}
    0, & e_{ij,t-1}=0,\\
    \text{sample from active relation posterior}, & e_{ij,t-1}=1, e_{ij,t}=1.\\
    \text{sample from relation marginal} & e_{ij,t-1}=1, e_{ij,t}=0
    \end{cases}
    \end{aligned} \end{equation} 
\end{enumerate}

Thus, the reverse sampler follows the semantic order
\begin{equation} \begin{aligned}
V_{t-1} \rightarrow E_{t-1} \rightarrow R_{t-1},
\end{aligned} \end{equation} 
and relation labels are sampled only on active directed edges.

\subsubsection{Structure-Aware Sampler Refinements}
\label{app:sampler_refinements}

The factorized reverse sampler defined in Sections~\ref{app:rev_sample}, \ref{sec:rev_sampler} produces samples by approximating the reverse diffusion distribution,
\begin{equation} \begin{aligned}
p_\theta(x_{t-1} \mid x_t),
\quad x_t = (V_t, E_t, R_t^+).
\end{aligned} \end{equation}
While this sampler respects the structural constraints of scene graphs, it may still produce locally inconsistent or semantically implausible configurations due to approximation errors in the learned model.
To address this, we introduce a family of \emph{structure-aware refinement kernels} that can be applied at inference time. These refinements operate as additional transition kernels that modify intermediate samples while preserving the constrained scene-graph state space.

Let the original, unrefined reverse sampling be, $y_{t-1} \sim p_\theta(x_{t-1} \mid x_t)$.
A refinement step applies a transition kernel
\begin{equation} \begin{aligned}
x_{t-1} \sim K_t(\cdot \mid y_{t-1}, x_t),
\end{aligned} \end{equation}
The resulting modified sampler is,
\begin{equation} \begin{aligned}
\tilde p_\theta(x_{t-1} \mid x_t)
=
\sum_{y} K_t(x_{t-1} \mid y, x_t)\, p_\theta(y \mid x_t).
\end{aligned} \end{equation}
We design refinement kernels that exploit the structured factorization of scene graphs to improve semantic consistency and long-tail relation modeling. All refinement kernels are constructed to preserve the constraint, $e_{ij} = 0 \Rightarrow r_{ij} = 0.$

$\blacksquare$ \textbf{Split Gibbs Refinement}

Following the success of split Gibbs samplers \cite{Vono_2019,chu2026split}, we consider a entity-wise Gibbs-style refinement over the factorized state, $x = (V, E, R^+).$ Let \(S \subseteq \{V, E, R^+\}\) be a subset of entities, and denote its complement by \(x_{\neg S}\). A Gibbs update replaces the selected subset by sampling from the conditional distribution over $k$ steps,
\begin{equation} \begin{aligned}
x_S^{(k+1)} \sim p_\theta(x_S \mid x_{\neg S}^{(k)}, x_t).
\end{aligned} \end{equation}

This defines the corresponding refinement kernel,
\begin{equation} \begin{aligned}
K_{\mathrm{Gibbs}}(x' \mid x)
=
p_\theta(x_S' \mid x_{\neg S}, x_t)\,
\mathbf{1}[x'_{\neg S} = x_{\neg S}].
\end{aligned} \end{equation}

For scene graphs, the choices are, $S \in \{V, E, R^+\},$
yielding updates such as
\begin{equation} \begin{aligned}
R^+ \sim p_\theta(R^+ \mid V, E, x_t),
\quad
E \sim p_\theta(E \mid V, R^+, x_t),
\quad
V \sim p_\theta(V \mid E, R^+, x_t).
\end{aligned} \end{equation}

In particular, relation resampling
\begin{equation} \begin{aligned}
r_{ij} \sim p_\theta(r_{ij} \mid v_i, v_j, e_{ij}=1, x_t)
\end{aligned} \end{equation}
directly enforces object–relation compatibility. Gibbs refinement corresponds to coordinate ascent algorithm \cite{wright2015coordinatedescentalgorithms}, $x_S^{(k+1)} = \arg\max_{x_S} p_\theta(x_S \mid x_{\neg S}^{(k)}, x_t).$
Thus, Gibbs iteratively moves the sample toward a local fixed point of the model’s conditional distributions:
\begin{equation} \begin{aligned}
x_S \approx \arg\max p_\theta(x_S \mid x_{\neg S}, x_t).
\end{aligned} \end{equation}

$\blacksquare$ \textbf{Soft Mask Refinement}

While Gibbs refinement updates entire entities, it may unnecessarily modify already correct entities and, therefore, might be too aggressive. To address this, we introduce a selective resampling mechanism based on model uncertainty.

Let \(M\) denote a subset of low-confidence entities. We temporarily mask them, $x_M \leftarrow \mathtt{MASK}$ while keeping the remaining entities fixed. The masked entities are then resampled:
\begin{equation} \begin{aligned}
x_M' \sim p_\theta(x_M \mid x_{\neg M}, x_M = \mathtt{MASK}, x_t).
\end{aligned} \end{equation}

This defines the mask refinement kernel,
\begin{equation} \begin{aligned}
K_{\mathrm{mask}}(x' \mid x)
=
p_\theta(x_M' \mid x_{\neg M}, \operatorname{mask}(x_M), x_t)\,
\mathbf{1}[x'_{\neg M} = x_{\neg M}].
\end{aligned} \end{equation}

To determine the entities to mask, confidence is derived from the model’s predictions. For a entitiy \(z\), we can define confidence based on entropy,
\begin{equation} \begin{aligned}
\lambda(z) = -\sum_c p_\theta(z = c \mid x_t)\log p_\theta(z = c \mid x_t).
\end{aligned} \end{equation}
We select, $M = \{ z : \lambda(z) < \tau \}$
or the top-\(k\) most uncertain entities.

For relations, masking is applied only on active edges:
\begin{equation} \begin{aligned}
M_R = \{(i,j): e_{ij}=1,\  \lambda(r_{ij}) < \tau_R\}.
\end{aligned} \end{equation}

Soft mask refinement, therefore, focuses on resampling uncertain entities, improving local consistency while preserving reliable structure.

$\blacksquare$ \textbf{Rare-Relation Refinement}

Finally, we introduce a refinement mechanism targeting long-tailed relation distributions.
Standard relation sampling uses, $p_\theta(r_{ij} \mid V, E, x_t),$
which tends to favor frequent relations. To counteract this, we define a tilted distribution,
\begin{equation}
    \begin{aligned}
\tilde p_\theta(r_{ij} \mid V, E, x_t)
\propto
p_\theta(r_{ij} \mid V, E, x_t)
\exp(\beta_{\mathrm{rare}}\, s(r_{ij})),        
    \end{aligned}
    \label{eqn:rare-tilt}
\end{equation}

where \(s(r)\) is a rarity score, $s(r) = -\log p_{\mathrm{data}}(r)$.


Unlike Gibbs refinement, which follows the scene graph factorization \(V \rightarrow E \rightarrow R\), rare-relation refinement inverts the dependency structure locally, $R \rightarrow E \rightarrow V.$, i.e., we sample rare relations from the rarety tilted distribution in Eqn. \ref{eqn:rare-tilt}, force edges, and then revise objects based on whether they match the rare relations sampled.
For a subset \(S_R\) of selected edges, we perform:
\begin{align*}
R_{S_R}' &\sim \tilde p_\theta(R_{S_R} \mid V, E, x_t), \\
E' &\sim p_\theta(E \mid V, R', x_t), \\
V' &\sim p_\theta(V \mid E', R', x_t).
\end{align*}

This defines the kernel
\begin{equation} \begin{aligned}
K_{\mathrm{rare}}(x' \mid x)
=
\tilde p_\theta(R_{S_R}' \mid V, E, x_t)\,
p_\theta(E' \mid V, R', x_t)\,
p_\theta(V' \mid E', R', x_t).
\end{aligned} \end{equation}

This reversed conditioning allows rare relations to induce compatible objects and edges.

$\blacksquare$ \textbf{Combined Refinement Strategy}

At each reverse timestep \(t\), we apply a composition of refinement kernels (or a subset of them):
\begin{equation} \begin{aligned}
x_{t-1}
\sim
K_{\mathrm{mask},t}
\circ
K_{\mathrm{rare},t}
\circ
K_{\mathrm{Gibbs},t}
\left( p_\theta(x_{t-1} \mid x_t) \right).
\end{aligned} \end{equation}

Refinements are applied selectively over different phases of the reverse process:
\begin{itemize}
    \item Gibbs refinement in early-to-mid timesteps to improve structural coherence,
    \item rare-relation tilting in mid timesteps to enhance long-tail relation generation (when needed).
    \item soft mask refinement in late timesteps to correct low-confidence entities.
\end{itemize}

They provide a flexible and structure-aware mechanism for improving sample quality without modifying the underlying diffusion model.

\subsection{Text-Conditioned Scene Graph Generation via Reward Tilting -- Additional Details}
\label{app:smc}
\paragraph{Sequential Monte Carlo (SMC) approximation.}
We approximate sampling from text-conditioned SG distribution, \(\tilde p(G_{0:T} \mid T)\) using Sequential Monte Carlo with \(D\) particles. At timestep \(t\), we maintain a set of $D$ particles, $\{G_t^{(d)}\}_{k=1}^D.$

\textbf{1. Particle Propagation step:}
\begin{equation} \begin{aligned}
G_{t-1}^{(d)} \sim p_\theta(G_{t-1} \mid G_t^{(d)}).
\end{aligned} \end{equation}

\textbf{2. Weight computation step:}
For each particle, we compute the predicted clean graph  and to distribute the reward across timesteps, incremental weights are used, $w_t=\exp\big(\beta (R_{t-1} - R_t)\big),$
\begin{equation} \begin{aligned}
\hat G_0^{(d)} = \epsilon_\theta(G_{t-1}^{(d)}, t-1),
\end{aligned} \end{equation}
\begin{equation} \begin{aligned}
w_{t-1}^{(d)}
=
\exp\big(\beta (R(\hat G_0^{(d)}, T) - R_t^{(d)})\big),
\end{aligned} \end{equation}
where \(R_t^{(d)} = R(\hat G_0^{(d)}(t), T)\) is the previous reward and $T$ is the text prompt (not to be confused with time $T$). Followed by normalization,
\begin{equation} \begin{aligned}
\tilde w_{t-1}^{(d)}
=
\frac{w_{t-1}^{(d)}}{\sum_{j=1}^D w_{t-1}^{(j)}}.
\end{aligned} \end{equation}

\textbf{3. Resampling Step:}
$\{G_{t-1}^{(d)}\}_{d=1}^D$ are the propagated particles with normalized weights
$\{\tilde w_{t-1}^{(d)}\}_{d=1}^D$, where $\sum_{d=1}^D \tilde w_{t-1}^{(d)} = 1$.
We perform resampling by drawing particle indices from a categorical distribution based on reward alignment,
\begin{equation} \begin{aligned}
I^{(d)} \sim \mathrm{Categorical}\big(\tilde w_{t-1}^{(1)}, \dots, \tilde w_{t-1}^{(D)}\big),
\end{aligned} \end{equation}
and setting
\begin{equation} \begin{aligned}
G_{t-1}^{(d)} \leftarrow G_{t-1}^{(I^{(d)})}.
\end{aligned} \end{equation}

Equivalently, this step samples particles from the empirical distribution, thereby replicating high-weight particles (high text similarity) and discarding low-weight ones.
\begin{equation} \begin{aligned}
\hat p(G_{t-1})
=
\sum_{d=1}^D \tilde w_{t-1}^{(d)}\, \delta_{G_{t-1}^{(d)}},
\end{aligned} \end{equation}
\paragraph{Reward function.}
\label{app:clip}
The reward \(R(G,T)\) measures compatibility between a scene graph \(G\) and a text prompt \(T\).
defined as,
\begin{equation} \begin{aligned}
R(G,T)
=
\cos\big(f_{\mathrm{CLIP}}(T),\ f_{\mathrm{CLIP}}(\mathrm{serialize}(G))\big),
\end{aligned} \end{equation}
where \(f_{\mathrm{CLIP}}\) is a pretrained CLIP encoder \cite{radford2021learning} and \(\mathrm{serialize}(G)\) converts the scene graph into a textual description.
Under the SMC procedure, the resulting samples approximate:
\begin{equation} \begin{aligned}
\tilde p(G_0 \mid T)
\propto
p_\theta(G_0)\,\exp\big(\beta R(G_0, T)\big),
\end{aligned} \end{equation}
thus enabling text-conditioned scene graph generation without retraining the diffusion model.

\subsection{Experimental Details}
\label{app:experiment}

\subsubsection{Metrics}

$\blacksquare$ \textbf{Graph Metrics.}

A scene graph is defined as $G=(V,E,R)$, where $V=\{v_i\}_{i=1}^{N_V}$ are node labels with $v_i \in \mathcal{V}$, $E=\{e_{ij}\}_{i,j=1}^{N_E}$ is binary edge existence with $e_{ij} \in \{0,1\}$, and $R=\{r_{ij}\}_{i,j=1}^{N_R}$ are relation labels with $r_{ij} \in \mathcal{R}$ and $r_{ij}=0$ indicating no relation. Let $p_{\text{data}}$ denote the empirical distribution of real graphs and $p_\theta$ denote the distribution of generated graphs.

\paragraph{Maximum Mean Discrepancy (MMD).}
We evaluate distributional similarity using Maximum Mean Discrepancy (MMD) over graph statistics \cite{you2018graphrnn,liao2019gran}. Given a feature map $\phi(G)$ and a positive-definite kernel (RBF) $k(\cdot,\cdot)$, MMD is defined as,
\begin{equation}
\mathrm{MMD}^2(\phi)
=
\mathbb{E}_{G,G' \sim p_{\text{data}}}[k(\phi(G), \phi(G'))]
+
\mathbb{E}_{\hat{G},\hat{G}' \sim p_\theta}[k(\phi(\hat{G}), \phi(\hat{G}'))]
-
2\mathbb{E}_{G \sim p_{\text{data}}, \hat{G} \sim p_\theta}[k(\phi(G), \phi(\hat{G}))].
\end{equation}

We instantiate $\phi(G)$ using different graph statistics:
\begin{itemize}
\item \textbf{Node MMD (N-MMD):} $\phi_V(G)$ is the histogram of node labels $v_i$.
\item \textbf{Relation MMD (R-MMD):} $\phi_R(G)$ is the histogram of relation labels over active edges ($e_{ij}=1$).
\item \textbf{In-degree MMD (ID-MMD):} $\phi_{\text{in}}(G)$ is the histogram of node in-degrees $d_i^{\text{in}} = \sum_j e_{ji}$.
\item \textbf{Out-degree MMD (OD-MMD):} $\phi_{\text{out}}(G)$ is the histogram of node out-degrees $d_i^{\text{out}} = \sum_j e_{ij}$.
\end{itemize}

\paragraph{Triplet-TV.}
We measure the discrepancy between distributions over <\texttt{object,relation,subject}> triplets. Considering only active edges ($e_{ij}=1$),
\begin{equation}
p(v_i, r, v_j)
=
\mathbb{P}(v_i, r_{ij}=r, v_j \mid e_{ij}=1).
\end{equation}

\begin{equation}
\mathrm{TV}_{\text{triplet}}
=
\frac{1}{2}
\sum_{v_i \in \mathcal{V}} \sum_{r \in \mathcal{R}} \sum_{v_j \in \mathcal{V}}
\left|
p_\theta(v_i, r, v_j)
-
p_{\text{data}}(v_i, r, v_j)
\right|.
\end{equation}

\paragraph{Attach-TV.}
We introduce a new metric, Attach-TV, which evaluates semantic attachment patterns given objects, by comparing relation distributions conditioned on object pairs. For active edges ($e_{ij}=1$),
\begin{equation}
p(r \mid v_i, v_j)
=
\mathbb{P}(r_{ij}=r \mid v_i, v_j, e_{ij}=1).
\end{equation}

\begin{equation}
\mathrm{TV}_{\text{attach}}
=
\frac{1}{2}
\sum_{v_i \in \mathcal{V}} \sum_{v_j \in \mathcal{V}}
p_{\text{data}}(v_i, v_j)
\sum_{r \in \mathcal{R}}
\left|
p_\theta(r \mid v_i, v_j)
-
p_{\text{data}}(r \mid v_i, v_j)
\right|.
\end{equation}

\paragraph{Tail-weighted Conditional TV (RARE-K-TV).}
To evaluate long-tailed relation modeling, we consider the distribution of relation labels conditioned on active edges, $p(r \mid e=1).$
We define a tail-weighted total variation distance,
\begin{equation}
\mathrm{TV}_{\text{tail}}
=
\frac{1}{2}
\sum_{r \in \mathcal{R}}
w_r
\left|
p_\theta(r \mid e=1)
-
p_{\text{data}}(r \mid e=1)
\right|.
\end{equation}
The weights emphasize rare predicates:
\begin{equation}
w_r \propto \frac{1}{p_{\text{data}}(r \mid e=1)^\alpha}, \quad \alpha \in [0.5,1],
\end{equation}
with normalization $\sum_r w_r = 1$.

$\blacksquare$ \textbf{Layout Metrics.}

From \cite{xu2024joint}, for a generated layout, we calculate the F1 score between this generated layout and every ground-truth layout. Assume there are $K$ object categories. Given a pair of generated and ground-truth layouts, the F1 score is calculated as $\text{F1} = \sum_{k \in K} w_k \text{F1}_k$, where $\text{F1}_k$ is the F1 score for that class $k$ and weight $w_k$. Calculating $\text{F1}_k$ is based on IoU metric. We use $10$ different bounding box IoU thresholds ranging from 0.05 to 0.5 with a step size of 0.05, $\text{F1}_k = \frac{1}{10}\sum_{iou \in [0.05:0.05:0.5]}\text{F1}(iou\mid k)$, where $\text{F1}(iou\mid k)$ means F1 score between two layouts given a specific object category $k$ and a IoU threshold $iou$. We calculate 4 different types of F1 scores: (1) F1-Vanilla (F1-V), where $w_k = \frac{1}{\mid K \mid}$ for every object category; (2) F1-Area (F1-A), where $w_k = \frac{\text{Area}(k)}{\sum_{k \in K}\text{Area}(k)}$ and $\text{Area}(k)$ is the average bounding box area in the validation set for the object category $k$; (3) F1-Frequency (F1-F), where$w_k = \frac{\text{Freq}(k)}{\sum_{k \in K}\text{Freq}(k)}$ and
$\text{Freq}(k)$ is the frequency of the object category $k$ in the validation set; (4) F1-Box, where the F1 calculation is purely based on the bounding box locations, that is, we treat all bounding boxes as having a single object category ($|K| = 1$ and $w_k = 1$). 

We compute results on $300$ data points from the CompSGBench dataset for Table \ref{tab:image_metrics_comp} as these metrics require expensive LLM inference to analyze the image and extract SG, objects, and relations.

\subsubsection{Implementation Details}
\label{app:training}
\paragraph{Denoising network.}
We use DiGress \cite{vignac2023digress} denoiser built on Graph transformer as the backbone but unlike DiGress we do not inject any global features. At timestep \(t\), the input graph state is
\[
x_t=(V_t,E_t,R_t),
\]
where \(V_t\) denotes object labels, \(E_t\), directed edge existence, and \(R_t\), relation labels. Given our factorized setup, we embed each entity separately,
\[
h_i^{V,0}=\mathrm{Emb}_V(v_{i,t})+\mathrm{Emb}_t(t),
\]
\[
h_{ij}^{E,0}=\mathrm{Emb}_E(e_{ij,t})+\mathrm{Emb}_t(t),
\]
\[
h_{ij}^{R,0}=
\begin{cases}
\mathrm{Emb}_R(r_{ij,t})+\mathrm{Emb}_t(t), & e_{ij,t}=1,\\
\mathbf{0}, & e_{ij,t}=0.
\end{cases}
\]
The edge-relation pair representation is initialized as,
\[
h_{ij}^{0}=\phi_E(h_{ij}^{E,0},h_{ij}^{R,0}),
\]
and the graph transformer updates node and pair edge-relation features jointly,
\[
(H^{V,L},H^{P,L})=\mathrm{GraphTransformer}_\theta(H^{V,0},H^{P,0},t).
\]
where $L = 1,\dots,\mathcal{L}$ is the transformer layers. 

\paragraph{Structured prediction heads.}
The factorized discrete space also requires modifying the prediction heads. 
The denoiser predicts clean graph entities with separate heads for objects and edges,
\[
p_\theta(v_{i,0}\mid x_t)=\mathrm{softmax}(W_V h_i^{V,L}),
\]
\[
p_\theta(e_{ij,0}\mid x_t)=\mathrm{softmax}(W_E h_{ij}^{P,L}),
\]

For predicting relations, we condition the relation head on the edge logits. Let $\ell_{ij}^E = W_E h_{ij}^{P,L}$ denote the edge logits, and let \(\phi(\cdot)\) be a learned projection of these logits. The relation distribution is defined as
\[
p_\theta(r_{ij,0} \mid x_t)
=
\mathrm{softmax}\!\left(
W_R \big[\, h_{ij}^{P,L},\ \phi(\ell_{ij}^E) \,\big]
\right).
\]

To enforce structural validity, we apply edge-gating at the output,
\[
p_\theta(r_{ij,0} \mid x_t)
=
\begin{cases}
\delta_0(r_{ij,0}), & e_{ij}=0,\\
\mathrm{softmax}\!\left(
W_R \big[\, h_{ij}^{P,L},\ \phi(\ell_{ij}^E) \,\big]
\right), & e_{ij}=1.
\end{cases}
\]

Thus, edge logits provide a soft conditioning signal for relation prediction during training, while hard gating ensures that no semantic relations are assigned to inactive edges.

\paragraph{Layout head.}
For generating grounded scene graphs and for downstream image generation, we optionally predict continuous layout bounding boxes from the final object-node representations using a layout head.
\[
\hat b_i = \sigma(W_B h_i^{V,L}) \in [0,1]^4.
\]

\paragraph{Training.}
For each training scene graph \(x_0=(V_0,E_0,R_0)\), we sample, $t\sim \mathrm{Uniform}\{1,\dots,T\}$
and construct, $x_t\sim q(x_t\mid x_0)$
using the structure-aware forward corruption process. The model is trained to reconstruct \(x_0\).
The base diffusion loss is cross-entropy loss between predictions $\{v,e,r\}_{ij,0}$ and clean entities $\{v,e,r\}^*_{ij,0}$,
\[
\mathcal{L}_{\mathrm{diff}}
=
\mathcal{L}_V+\lambda_E\mathcal{L}_E+\lambda_R\mathcal{L}_R.
\]

The object loss and directed edge loss are,
\[
\mathcal{L}_V
=
-\frac{1}{\sum_i m_i}
\sum_i m_i
\log p_\theta(v_{i,0} = v^*_{i,0}\mid x_t) \quad
\mathcal{L}_E
=
-\frac{1}{\sum_{i\neq j}M_{ij}}
\sum_{i\neq j}M_{ij}
\log p_\theta(e_{ij,0} = e^*_{ij,0}\mid x_t)
\]


The relation loss is computed only on active clean edges:
\[
\mathcal{L}_R
=
-\frac{1}{\sum_{i\neq j}M_{ij}e_{ij,0}}
\sum_{i\neq j}M_{ij}e_{ij,0}
w_{r_{ij,0}}
\log p_\theta(r_{ij,0} = r^*_{ij,0}\mid x_t).
\]

To handle the long-tailed relation distribution, we use relation class weights:
\[
w_r=\frac{1}{(p_{\mathrm{data}}(r)+\epsilon)^\alpha},
\qquad \alpha\in[0,1].
\]

For layout head supervision and performing grounding-aware training,
\[
\mathcal{L}_{\mathrm{box}}
=
\frac{1}{\sum_i m_i}
\sum_i m_i
\left[
\|\hat b_i-b_i\|_1
+
\lambda_{\mathrm{giou}}
\big(1-\mathrm{GIoU}(\hat b_i,b_i)\big)
\right].
\]

The total objective is $\mathcal{L}
=
\mathcal{L}_{\mathrm{diff}}
+
\lambda_{\mathrm{box}}\mathcal{L}_{\mathrm{box}}.$

Optionally, we also include intermediate reverse-step supervision: 
\begin{equation} \begin{aligned}
\mathcal{L}_{\mathrm{rev}}
=
\lambda_{\mathrm{rev},V}\mathcal{L}_{\mathrm{rev},V}
+
\lambda_{\mathrm{rev},E}\mathcal{L}_{\mathrm{rev},E}
+
\lambda_{\mathrm{rev},R}\mathcal{L}_{\mathrm{rev},R},
\end{aligned} 
\label{eqn:training_loss}
\end{equation} 
where the target is \(x_{t-1}\) sampled from the known forward process. Finally,
\[
\mathcal{L}
=
\mathcal{L}_{\mathrm{diff}}
+
\lambda_{\mathrm{box}}\mathcal{L}_{\mathrm{box}}
+
\lambda_{\mathrm{rev}}\mathcal{L}_{\mathrm{rev}}.
\]

\paragraph{Optimization.}
We optimize the model with AdamW, LR \(10^{-4}\), weight decay \(10^{-2}\), gradient clipping, and an exponential moving average of model parameters for sampling. We use the training data distribution as the prior for the forward noising process instead of a uniform distribution. The training is performed on a single NVIDIA A6000 GPU (48GB). The number of diffusion steps for sampling is \(T=100\). The training and sampling scripts are available here\footnote{https://annonymous4565.github.io/annon/}.

\subsection{Ablation Studies}
\label{app:ablation}
In this section, we provide additional ablation studies. 

(1) \emph{Sampler Refinements:} We ablate over various sampler refinements discussed in Secs. \ref{sec:rev_sampler} \& \ref{app:sampler_refinements} and their schedules. Table \ref{tab:ablation_5} (on VG data) confirms our understanding that Gibbs refinement is beneficial at early timesteps but not too early as it may be too noisy to effect any improvements. Soft mask refinement is best as a last-mile refinement strategy to quickly resample uncertain predictions. Rare refinement, on the other hand, degrades graph performance at the expense of rare relation sampling. Thus, we use early Gibbs and late soft mask refinements as the final setting.

(2) \emph{Entity training loss scales:} We perform ablations on different per-entity loss scales i.e., $\lambda_{O,E,R}$ in the training loss (Eqn. \ref{eqn:training_loss}) on the VG dataset in Table \ref{tab:ablation_4}. Best performance is observed for $\lambda_O, \lambda_E, \lambda_R = (1,0.5,0.5)$.

(3) \emph{Rare-relation Class Weighting:} Table \ref{tab:ablation_2} ablates on the VG dataset over different class weighting strategies such as, (1) simple $w_k = 1$ (2) Effective num, $w_k = \frac{1-\beta}{1-\beta^{f_k}}$, $\beta \in (0,1)$ and (3) inverse frequency, $w_k = \frac{1}{f_k}$, where $k$ is a relation class and $f_k$ is class frequency. The effective num approach performs the best, as it is based on diminishing emphasis as class count increases, while inverse frequency is too aggressive.

\begin{table}[H]
\centering
\caption{Ablation study over different sampling refinement schedules (start - $T_{\text{Gibbs}}, T_{\text{mask}}, T_{\text{rare}}$). Each refinement is applied for $\Delta = 10$ timesteps.}
\small
\begin{tabular}{@{}ccccccccc@{}}
\toprule
$T_{\text{Gibbs}}$ & $T_{\text{mask}}$ & $T_{\text{rare}}$ & N-MMD & R-MMD & ID-MMD & OD-MMD & TRIP-TV & RARE-K-TV \\
\midrule
10  & -  & -  & 0.0048 & 0.0058 & 0.0043 & 0.0047 & 0.5367 & 0.5310 \\
25  & -  & -  & 0.0046 & 0.0058 & 0.0043 & 0.0047 & 0.5288 & 0.5218 \\
50  & -  & -  & 0.0055 & 0.0063 & 0.0054 & 0.0053 & 0.6182 & 0.5536 \\
-   & 50 & -  & 0.0053 & 0.0058 & 0.0044 & 0.0048 & 0.5560 & 0.5514 \\
-   & 75 & -  & 0.0051 & 0.0060 & 0.0045 & 0.0048 & 0.5548 & 0.5466 \\
-   & 90 & -  & 0.0051 & 0.0060 & 0.0044 & 0.0047 & 0.5479 & 0.5347 \\
25  & 90 & 50 & 0.0175 & 0.0424 & 0.0053 & 0.0044 & 0.7235 & 0.2627 \\
\textbf{25}  & \textbf{90} & \textbf{-}  & \textbf{0.0046} & \textbf{0.0058} & \textbf{0.0043} & \textbf{0.0046} & \textbf{0.4920} & \textbf{0.5194} \\
\bottomrule
\end{tabular}
\label{tab:ablation_5}
\end{table}

\begin{table}[H]
\centering
\caption{Ablation study on different per-entity ($V,R,E)$ loss scales.}
\small
\begin{tabular}{@{}lccccccc@{}}
\toprule
Node loss & Relation loss & Edge loss & N-MMD & R-MMD & ID-MMD & OD-MMD & TRIP-TV \\
\midrule
1   & 1   & 1   & 0.0012 & 0.0091 & 0.0064 & 0.0060 & 0.5882 \\
0.5 & 1   & 1   & 0.0045 & 0.0090 & 0.0065 & 0.0060 & 0.5950 \\
1   & 0.5 & 1   & 0.0011 & 0.0095 & 0.0068 & 0.0062 & 0.5789 \\
1   & 1   & 0.5 & 0.0011 & 0.0096 & 0.0068 & 0.0062 & 0.5802 \\
0.5 & 0.5 & 1   & 0.0012 & 0.0095 & 0.0068 & 0.0062 & 0.5940 \\
\textbf{1} & \textbf{0.5} & \textbf{0.5} & \textbf{0.0046} & \textbf{0.0058} & \textbf{0.0043} & \textbf{0.0046} & \textbf{0.4920} \\
0.5 & 1   & 0.5 & 0.0012 & 0.0095 & 0.0065 & 0.0061 & 0.5906 \\
\bottomrule
\end{tabular}
\label{tab:ablation_4}
\end{table}

\begin{table}[H]
\centering
\caption{Ablation study on different relation weighting strategy for long-tail relations.}
\small
\begin{tabular}{@{}lcccccc@{}}
\toprule
Rel-K class weight & N-MMD & R-MMD & ID-MMD & OD-MMD & TRIP-TV & RARE-K-TV\\
\midrule
Simple        & 0.0024 & 0.0828 & 0.0155 & 0.0155 & 0.7422 & 0.5802\\
Inverse freq & 0.0088 & 0.0105 & 0.0091 & 0.0093 & 0.5257 & 0.5437\\
\textbf{Effective\_num}   & \textbf{0.0046} & \textbf{0.0058} & \textbf{0.0043} & \textbf{0.0046} & \textbf{0.4920} & \textbf{0.5194}\\
\bottomrule
\end{tabular}
\label{tab:ablation_2}
\end{table}


\subsection{Scene Graph Generation -- Qualitative results}

\subsubsection{Sampled Scene Graphs}
Figures \ref{fig:uncond_vg}, \ref{fig:uncond_vg_1} are graphs sampled from {\name} trained on Visual Genome data. Figure \ref{fig:uncond_coco} is correspondingly for COCO, and Figure \ref{fig:uncond_comp} shows graphs sampled from {\name} trained on LAION-SG data.

\begin{figure*}[h]
    \centering
    \includegraphics[width=1\columnwidth]{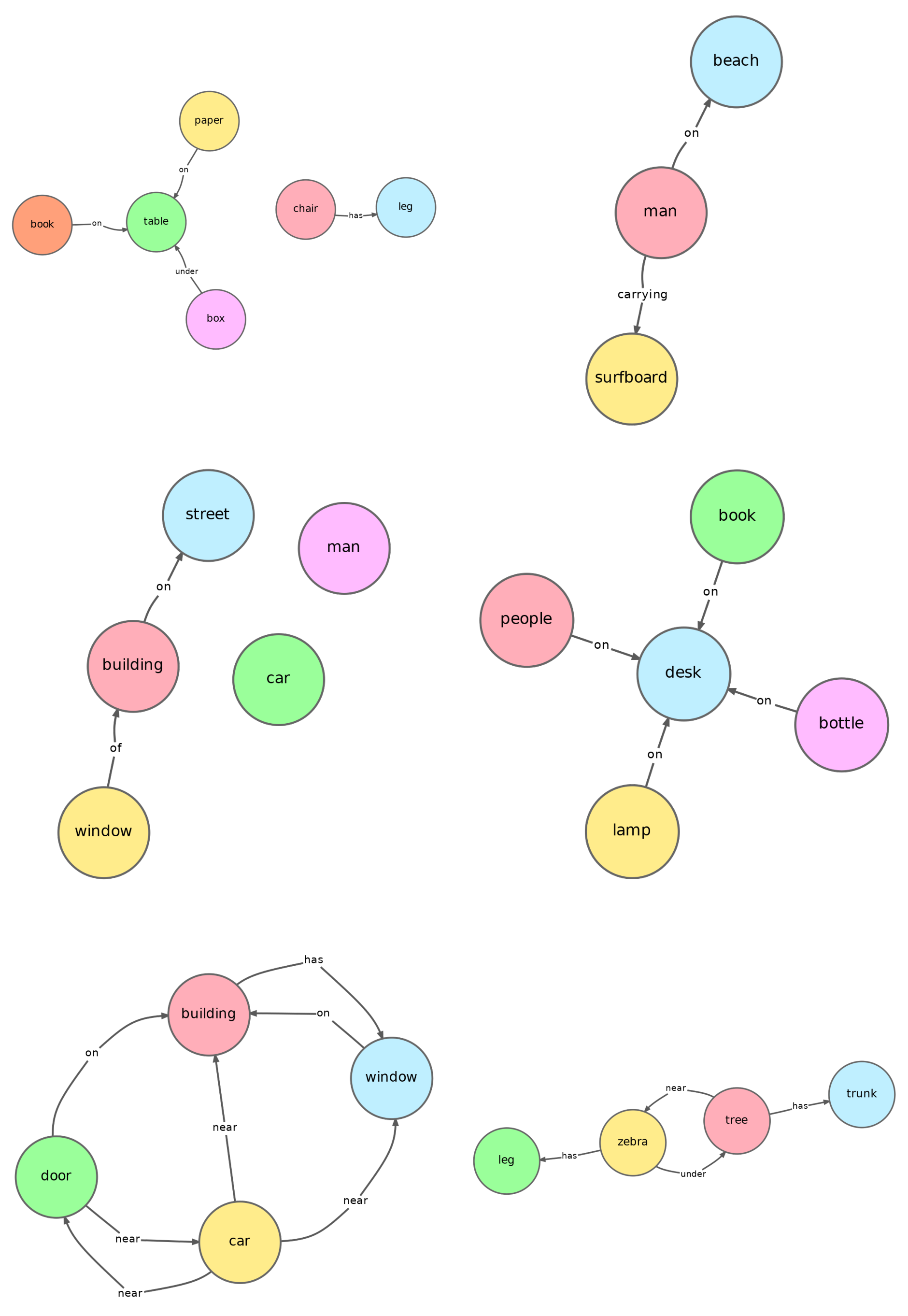}
    \caption{Qualitative Results: Scene graphs sampled from {\name} trained on Visual Genome \cite{krishna2016visualgenomeconnectinglanguage}.}
    \label{fig:uncond_vg}
    \vspace{-1.7em}
\end{figure*}

\begin{figure*}[h]
    \centering
    \includegraphics[width=1\columnwidth]{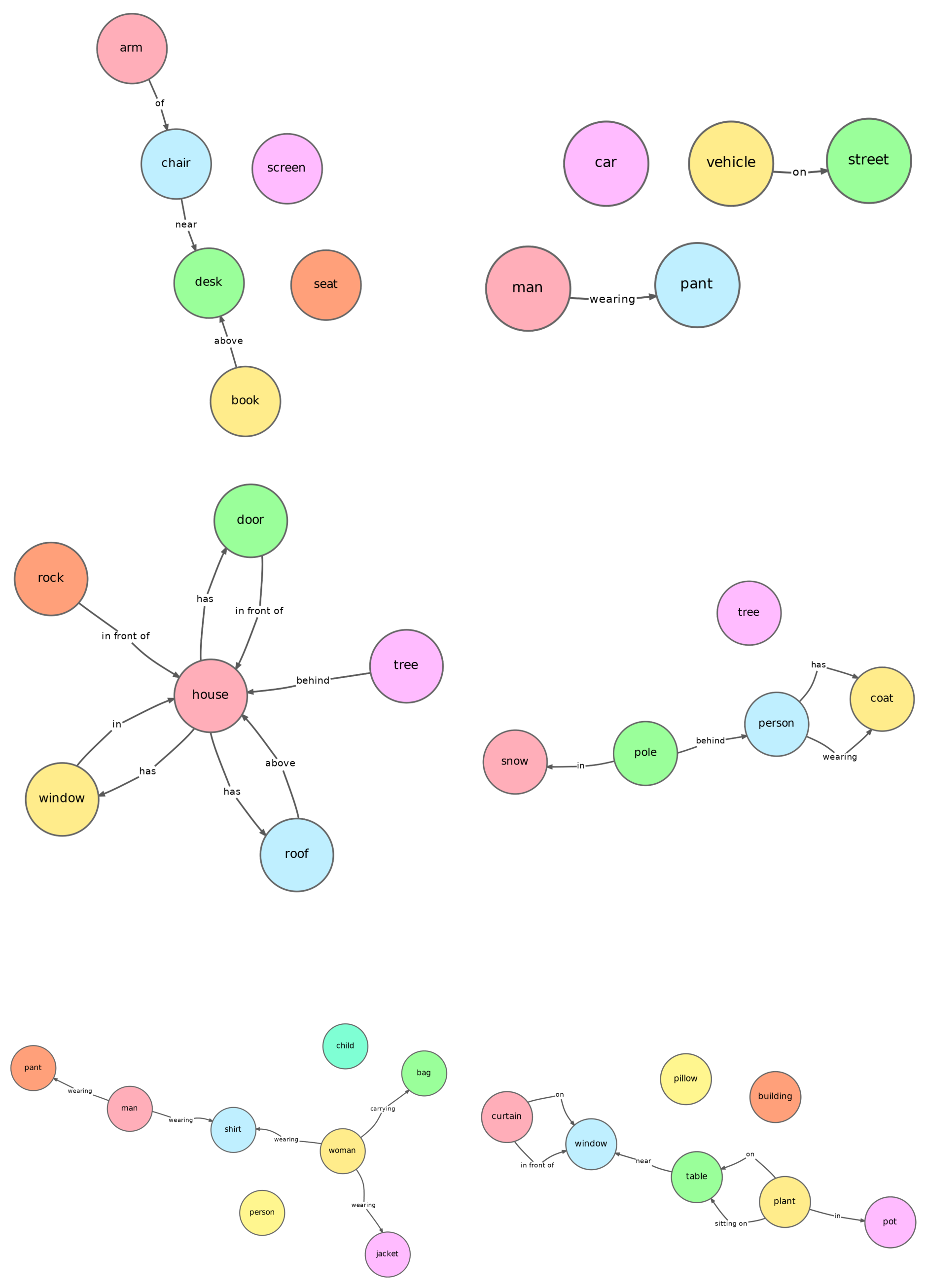}
    \caption{Qualitative Results: Scene graphs sampled from {\name} trained on Visual Genome.}
    \label{fig:uncond_vg_1}
    \vspace{-1.7em}
\end{figure*}

\begin{figure*}[h]
    \centering
    \includegraphics[width=1\columnwidth]{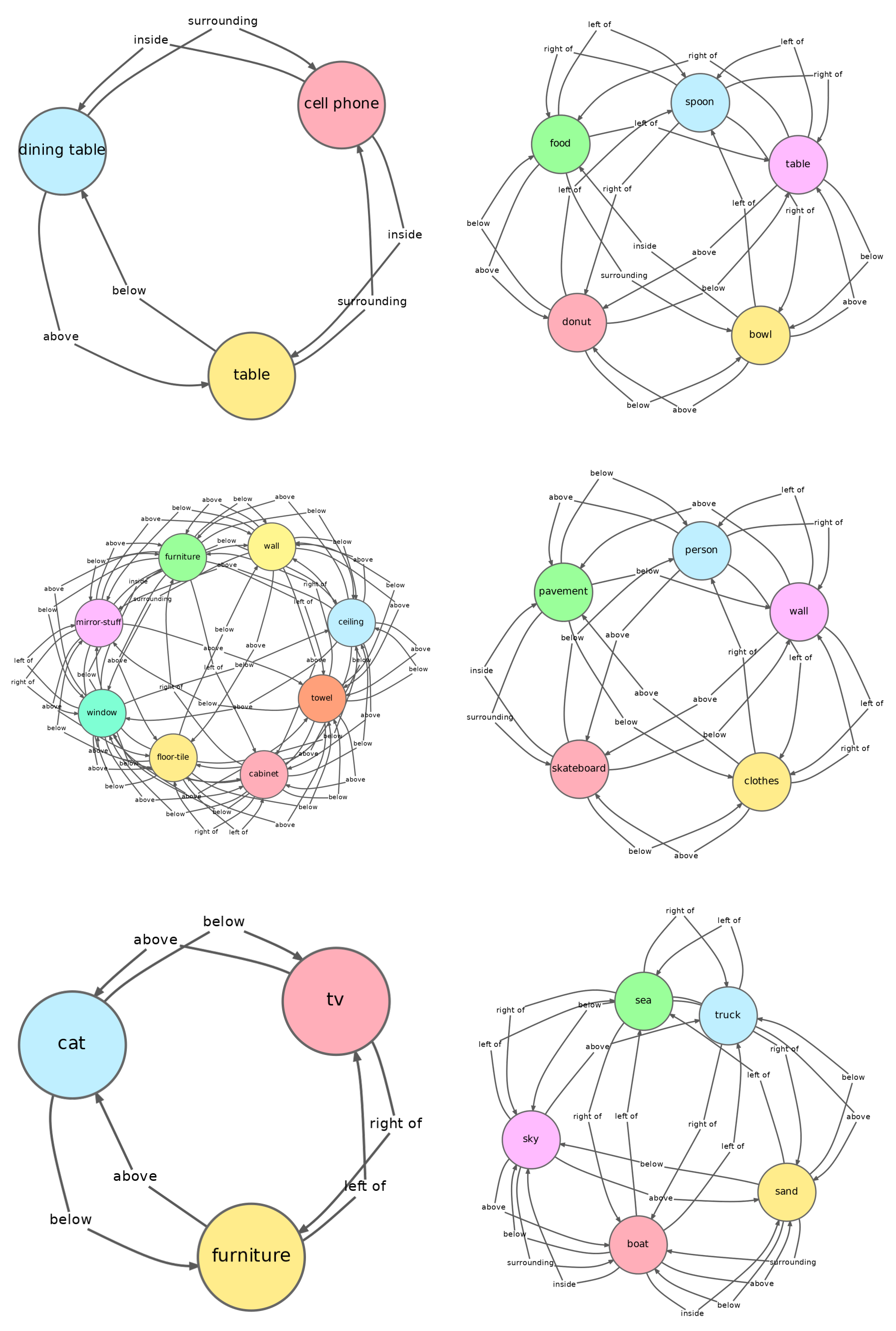}
    \caption{Qualitative Results: Scene graphs sampled from {\name} trained on COCO dataset \cite{lin2015microsoftcococommonobjects}. \textit{Note: The COCO dataset represents scenes with symmetric edges.}}
    \label{fig:uncond_coco}
    \vspace{-1.7em}
\end{figure*}

\begin{figure*}[h]
    \centering
    \includegraphics[width=1\columnwidth]{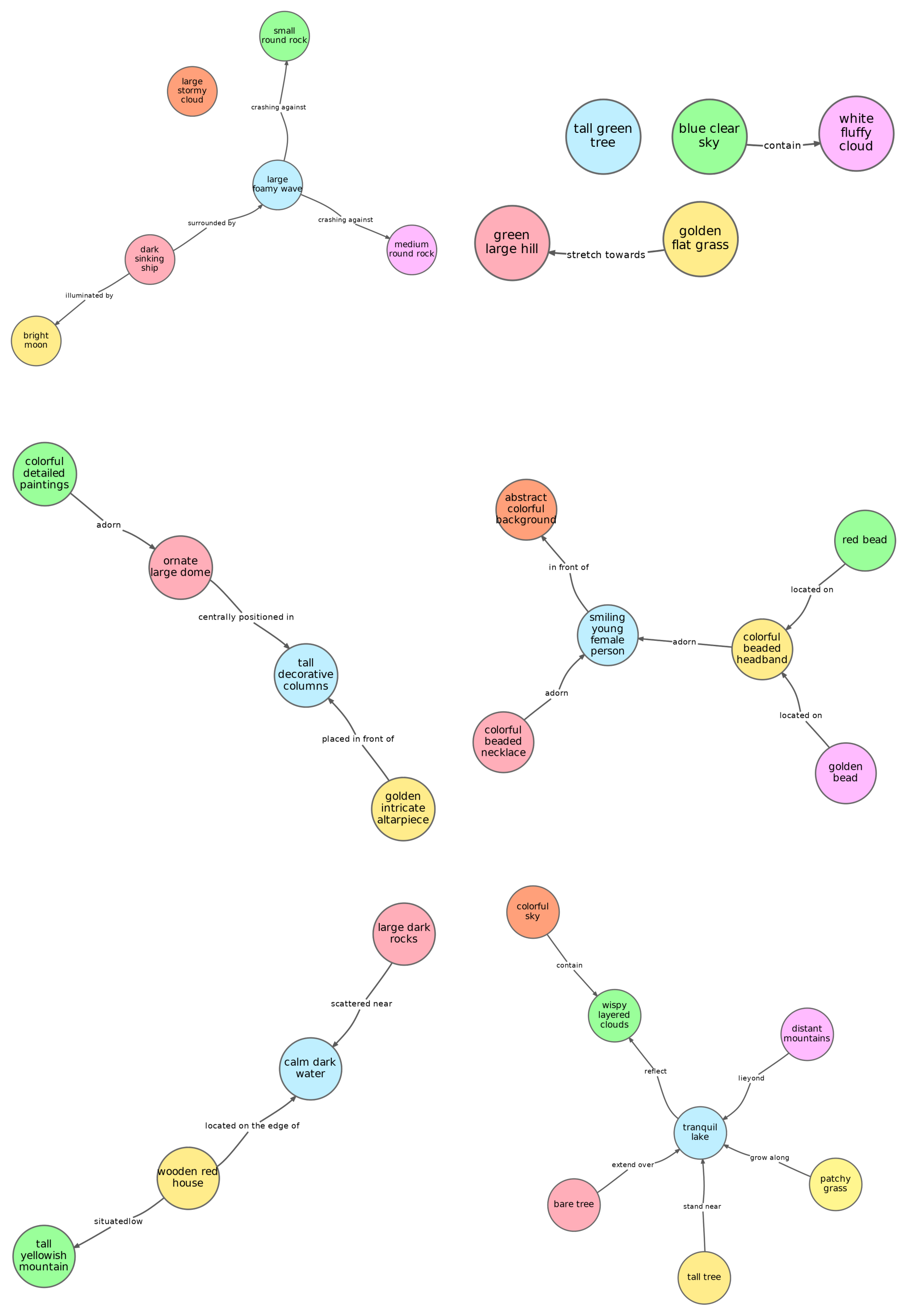}
    \caption{Qualitative Results: Scene graphs sampled from {\name} trained on LAION-SG \cite{li2024laionsgenhancedlargescaledataset}.}
    \label{fig:uncond_comp}
    \vspace{-1.7em}
\end{figure*}

\subsubsection{Graph Completion Tasks}
Figures \ref{fig:comple_vg}, \ref{fig:comple_vg_1} illustrate qualitative results for the single object and single relation completion task. Likewise, Figures \ref{fig:comple_comp} and \ref{fig:comple_comp_1} illustrate completion examples on CompSGBench benchmark data.

\begin{figure*}[h]
    \centering
    \includegraphics[width=1\columnwidth]{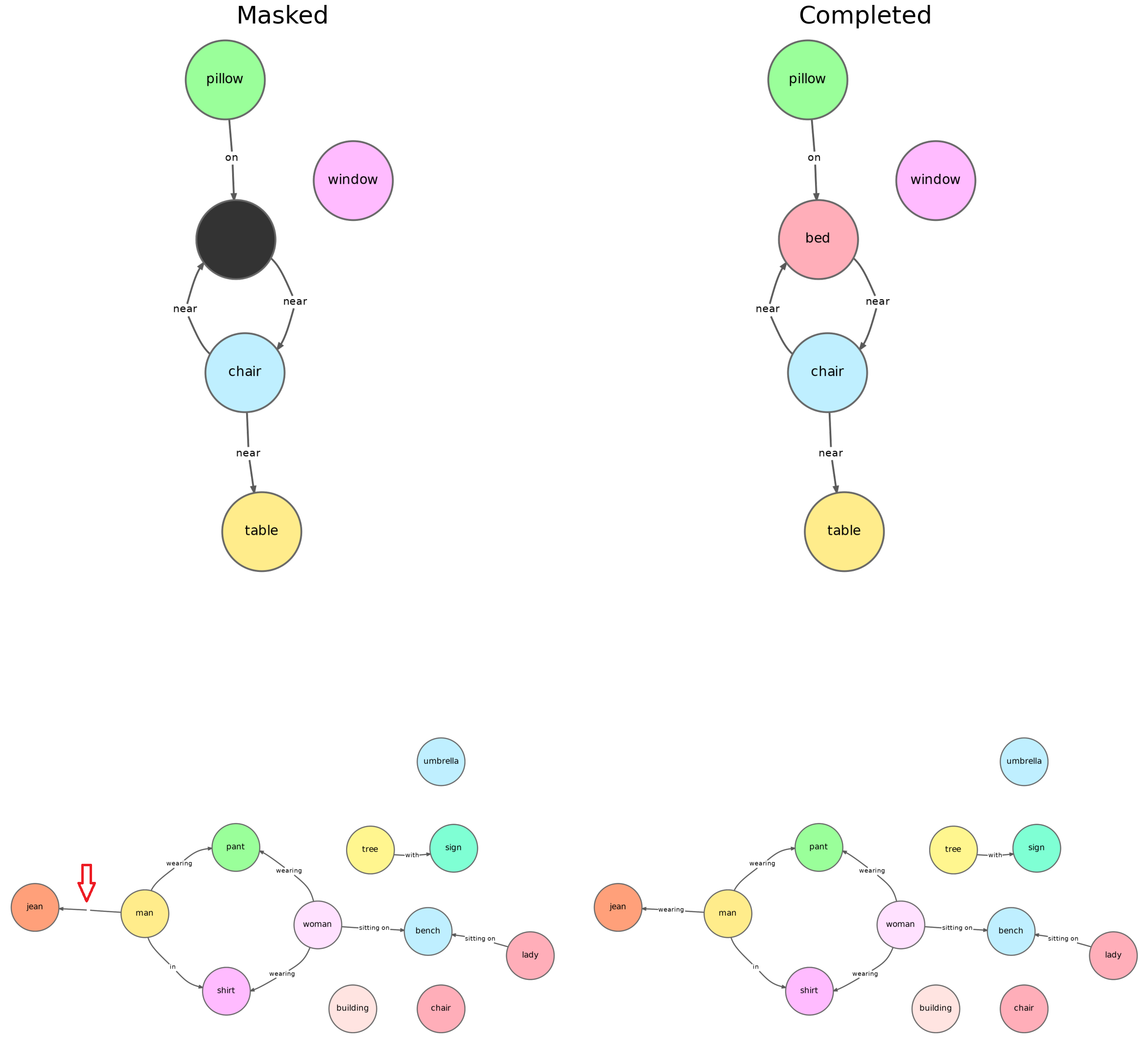}
    \caption{Qualitative Results: {\name} scene graph single object/relation completion task on Visual Genome. Masked objects are in black, and masked relations are indicated by red arrows.}
    \label{fig:comple_vg}
    \vspace{-1.7em}
\end{figure*}

\begin{figure*}[h]
    \centering
    \includegraphics[width=1\columnwidth]{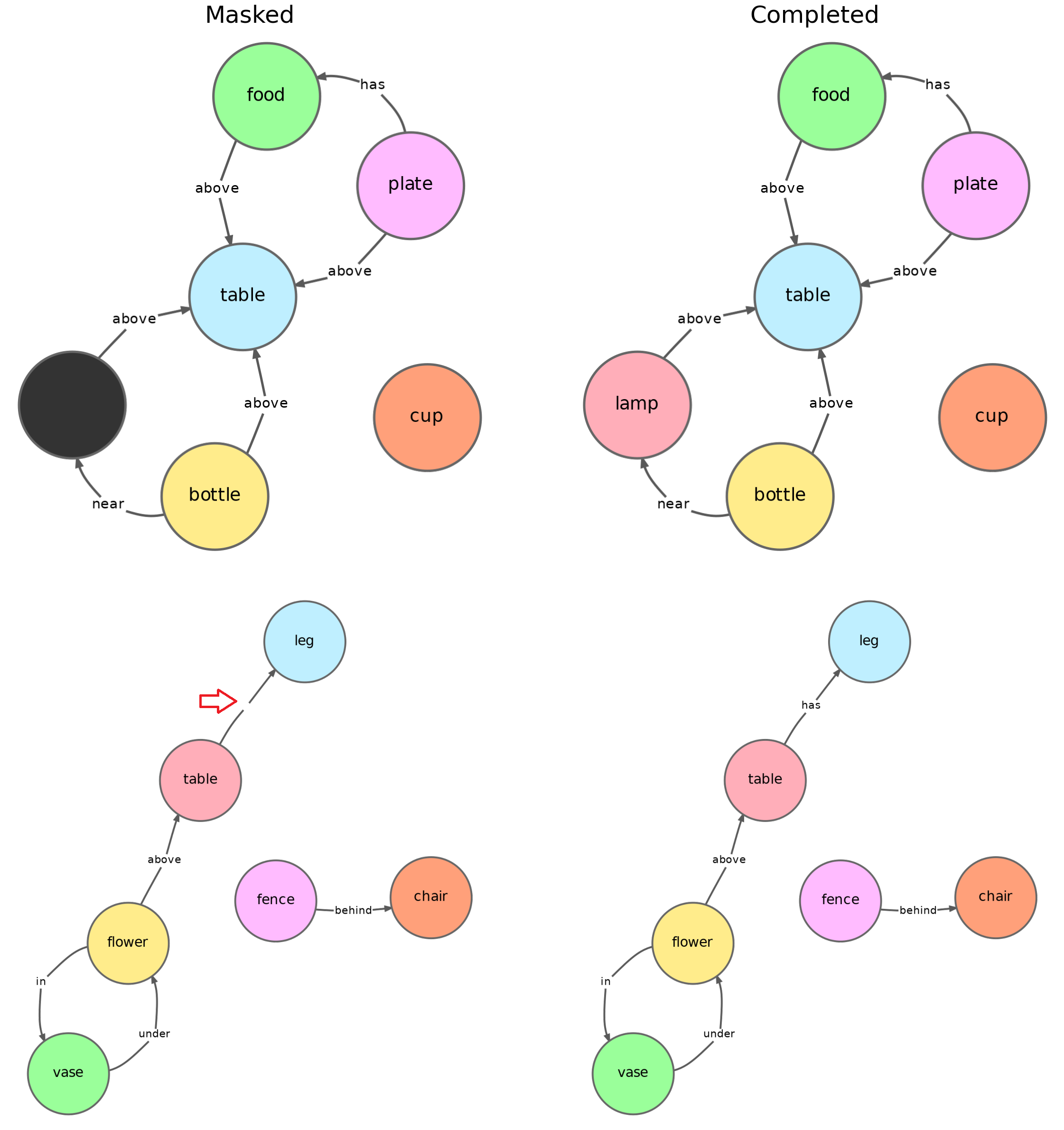}
    \caption{Qualitative Results: {\name} scene graph single object/relation completion task on Visual Genome dataset. Masked objects are in black, and masked relations are indicated by red arrows.}
    \label{fig:comple_vg_1}
    \vspace{-1.7em}
\end{figure*}

\begin{figure*}[h]
    \centering
    \includegraphics[width=1\columnwidth]{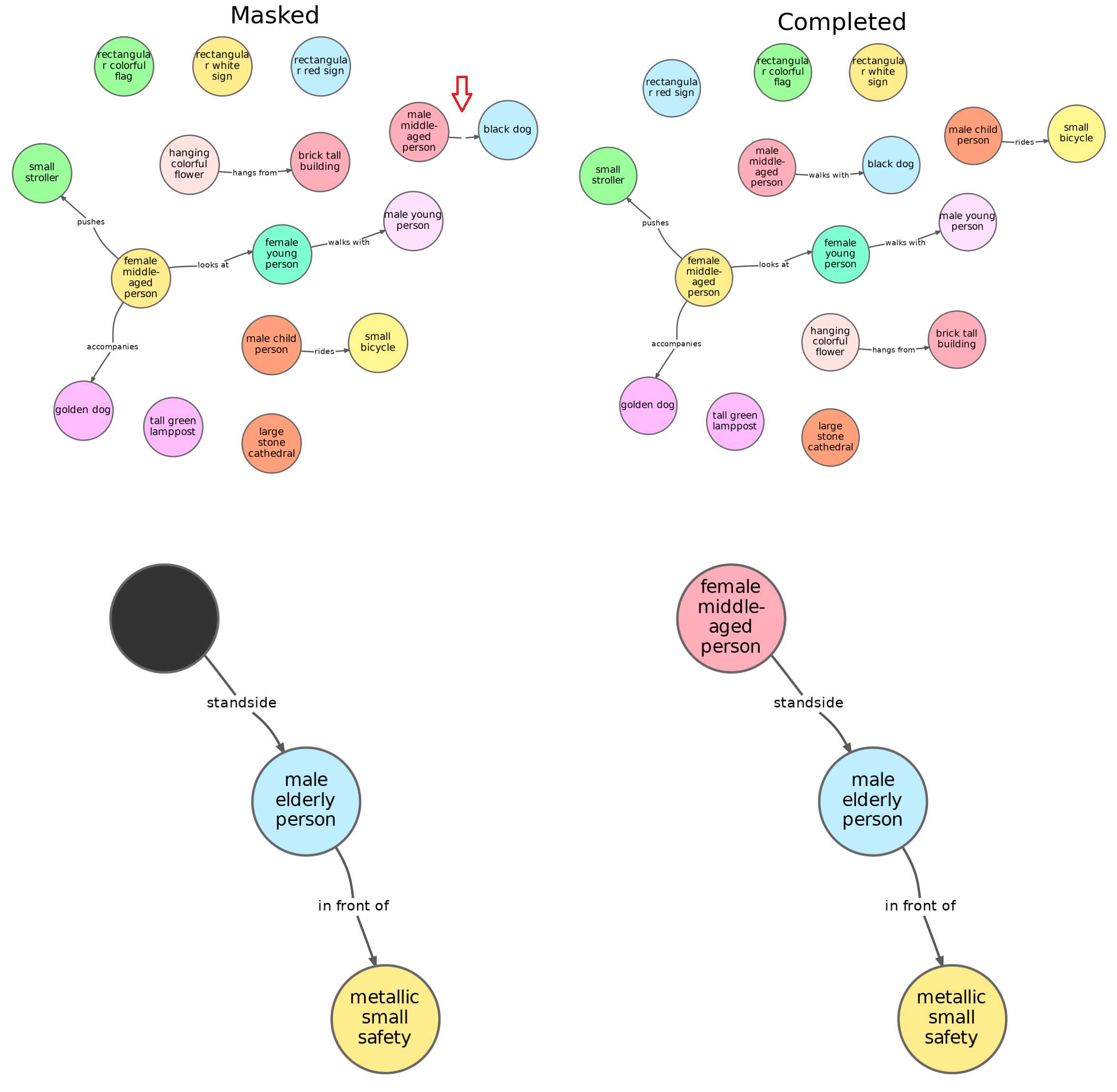}
    \caption{Qualitative Results: {\name} scene graph single object/relation completion task on CompSGBench \cite{li2024laionsgenhancedlargescaledataset}. Masked objects are in black, and masked relations are indicated by red arrows.}
    \label{fig:comple_comp}
    \vspace{-1.7em}
\end{figure*}

\begin{figure*}[h]
    \centering
    \includegraphics[width=1\columnwidth]{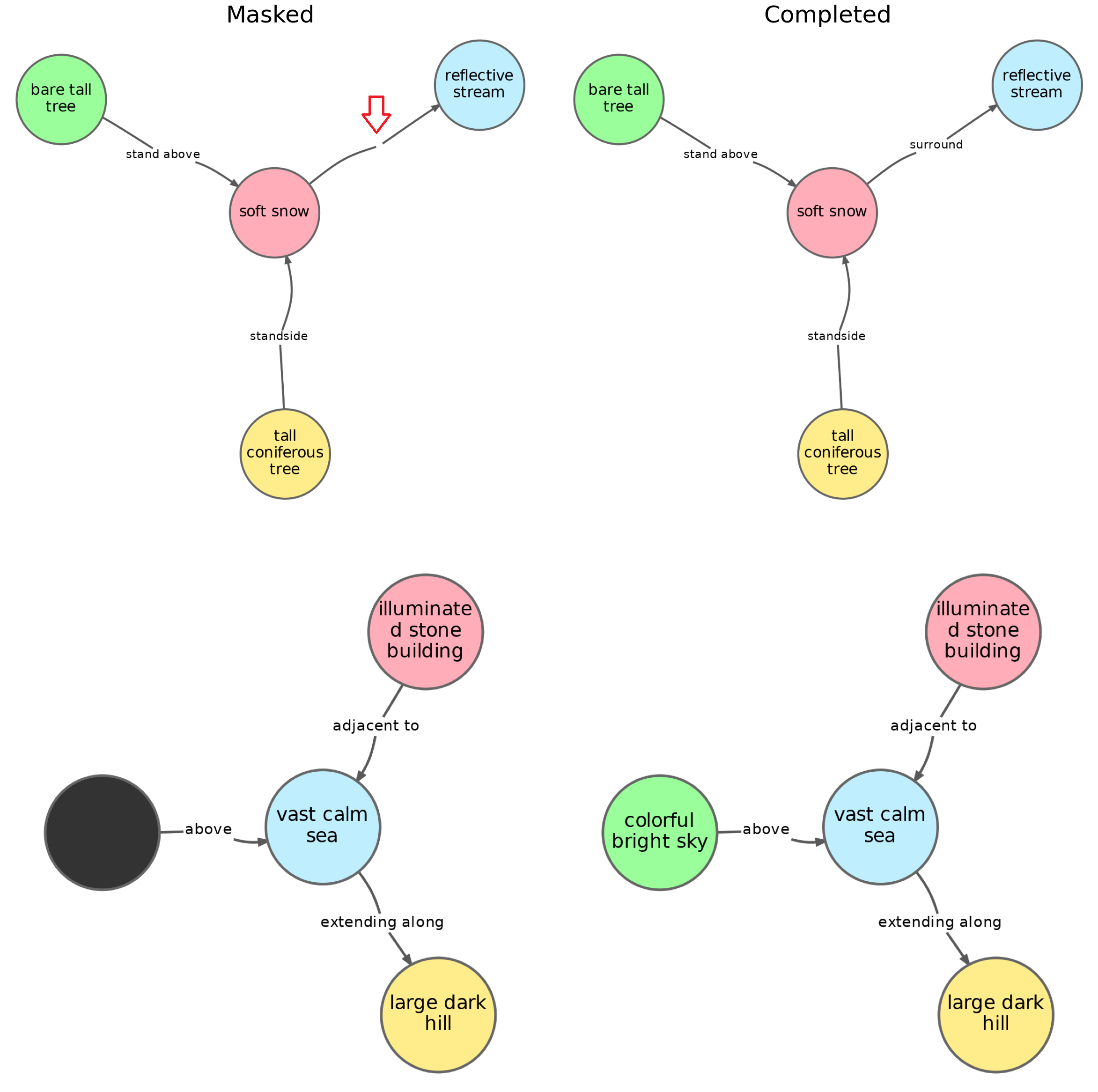}
    \caption{Qualitative Results: {\name} scene graph single object/relation completion task on CompSGBench. Masked objects are in dark gray, and masked relations are indicated with red arrows.}
    \label{fig:comple_comp_1}
    \vspace{-1.7em}
\end{figure*}

\subsection{Text-Image Generation -- Qualitative Results}
\subsubsection{Text-to-SG-to-Image Generation}
\label{app:t2i_qual_res}
Figure \ref{fig:t2i_comp_sg} shows the generated text-conditioned scene graphs, which were then used for graph-conditioned image generation in Figure \ref{fig:t2i_comp} in the main text.


Figures \ref{fig:t2i_comp_1}, \ref{fig:t2i_comp_2} illustrate additional qualitative examples of text-to-image generation, comparing text-based composition methods: SDXL, ComposedDiffusion, CO3, and ours, {\name}. Figures  \ref{fig:t2i_comp_sg_1}, \ref{fig:t2i_comp_sg_2} show the corresponding text-conditioned scene graphs sampled from {\name}, which were in turn used to generate the graph-conditioned images.

\begin{figure*}[t]
    \centering
    \includegraphics[width=1\columnwidth]{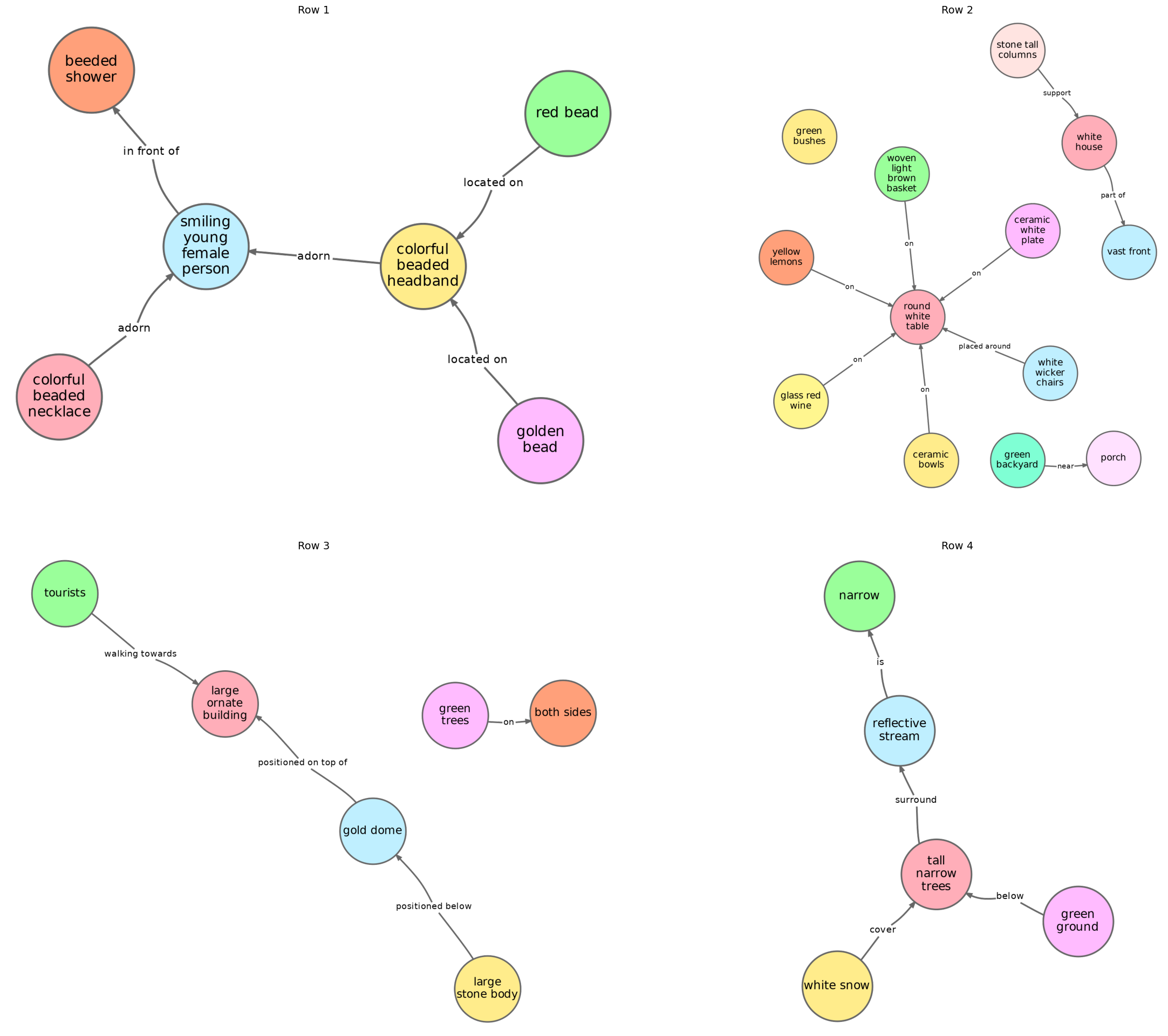}
    \caption{Qualitative Results: Sampled text-conditioned scene graphs used for generating image results of {\name} in Figure \ref{fig:t2i_comp} (Main text).}
    \label{fig:t2i_comp_sg}
    \vspace{-1.7em}
\end{figure*}

\begin{figure*}[t]
    \centering
    \includegraphics[width=0.8\columnwidth]{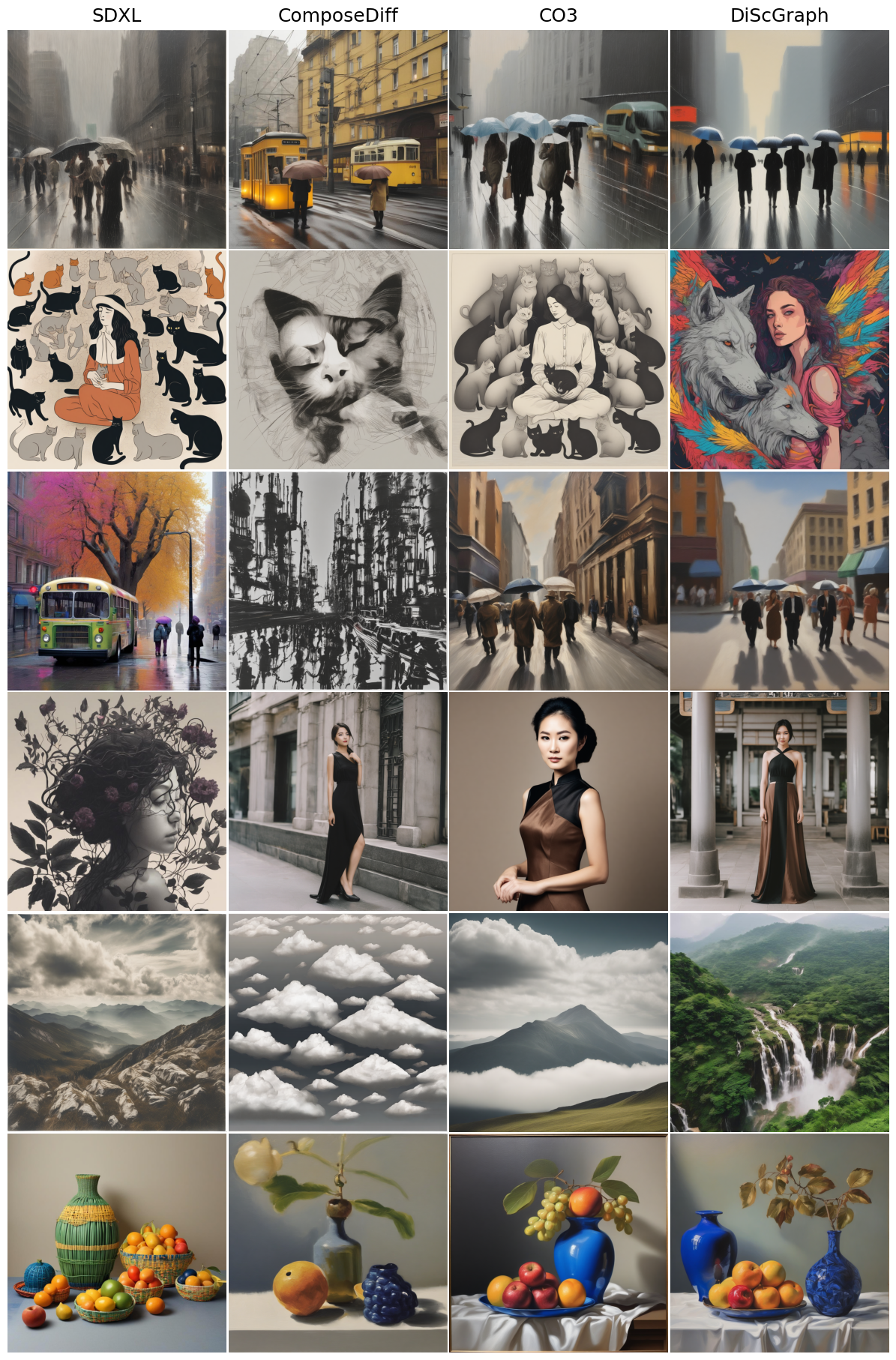}
    \caption{Qualitative Results: T2I generation comparison of methods, SDXL, ComposeDiff, CO3, and {\name}. For prompts:\textit{  ``Five people are walking on a wet street with buildings. Two people are wearing skirts and carrying blue umbrellas while the others are wearing pants and long coats while carrying umbrellas. It is gloomy with some bright illumination from lights", ``a girl has wolves near her and she has colorful background around her with feathers of different colors.", ``Four people consisting of two men and women are walking down a street which is surrounded by buildings on sides. One man is old and another man is young. The women are mother and daughter. They are carrying blue umbrellas.'', ``An Asian woman in a black and gold dress is standing in front of a large stone archway and is posing for a camera.'',``a far away perspective shot of a mountainous area with lush green trees and a dramatic waterfall in the middle.", ``An oil painting of various yellow and orange fruits on a blue plate that is on a table with table cloth. Two blue vases are on the table.",}}
    \label{fig:t2i_comp_1}
    \vspace{-1.7em}
\end{figure*}

\begin{figure*}[t]
    \centering
    \includegraphics[width=1\columnwidth]{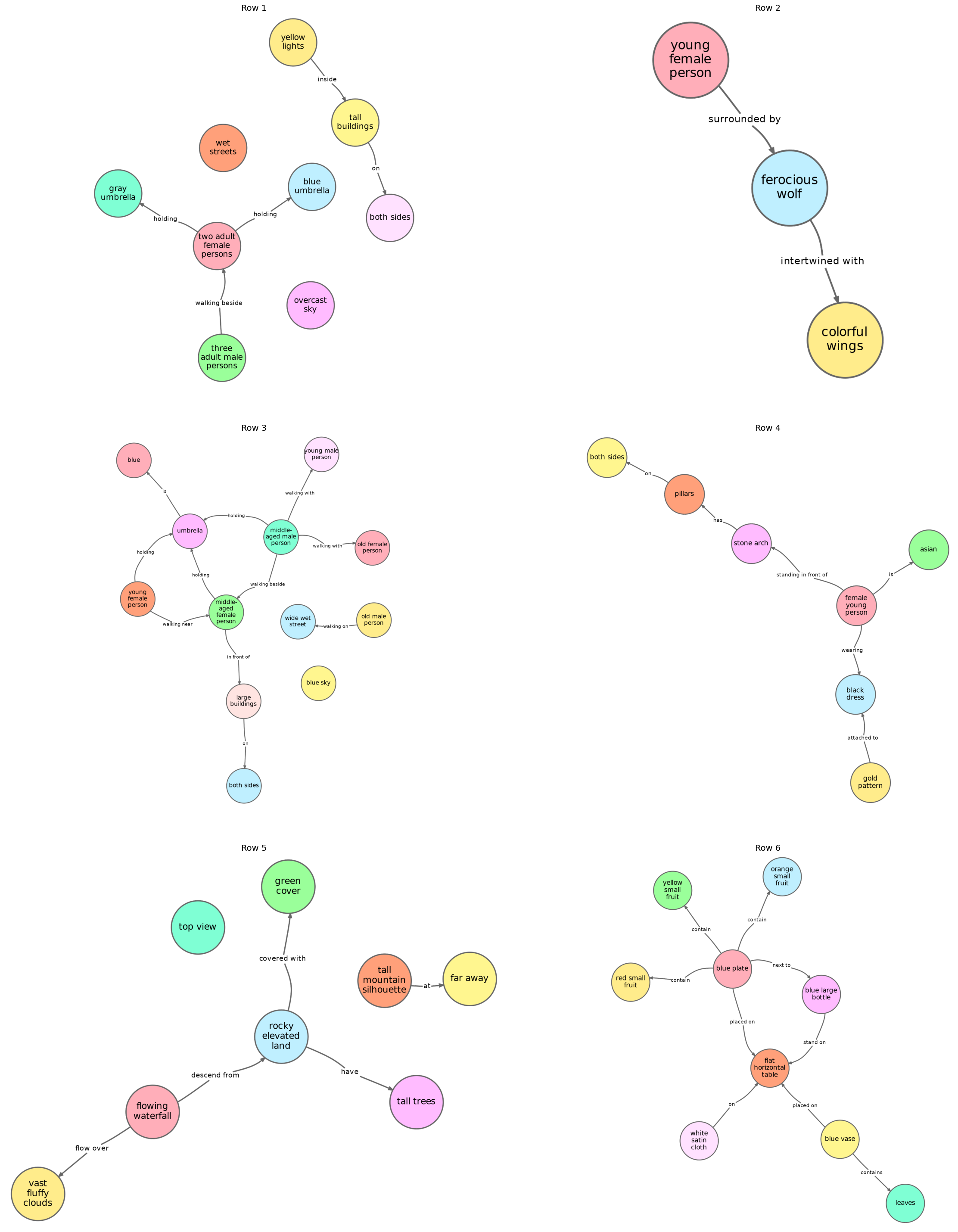}
    \caption{Qualitative Results: Sampled text-conditioned scene graphs used for generating image results of {\name} in Figure \ref{fig:t2i_comp_1}}
    \label{fig:t2i_comp_sg_1}
    \vspace{-1.7em}
\end{figure*}

\begin{figure*}[t]
    \centering
    \includegraphics[width=1\columnwidth]{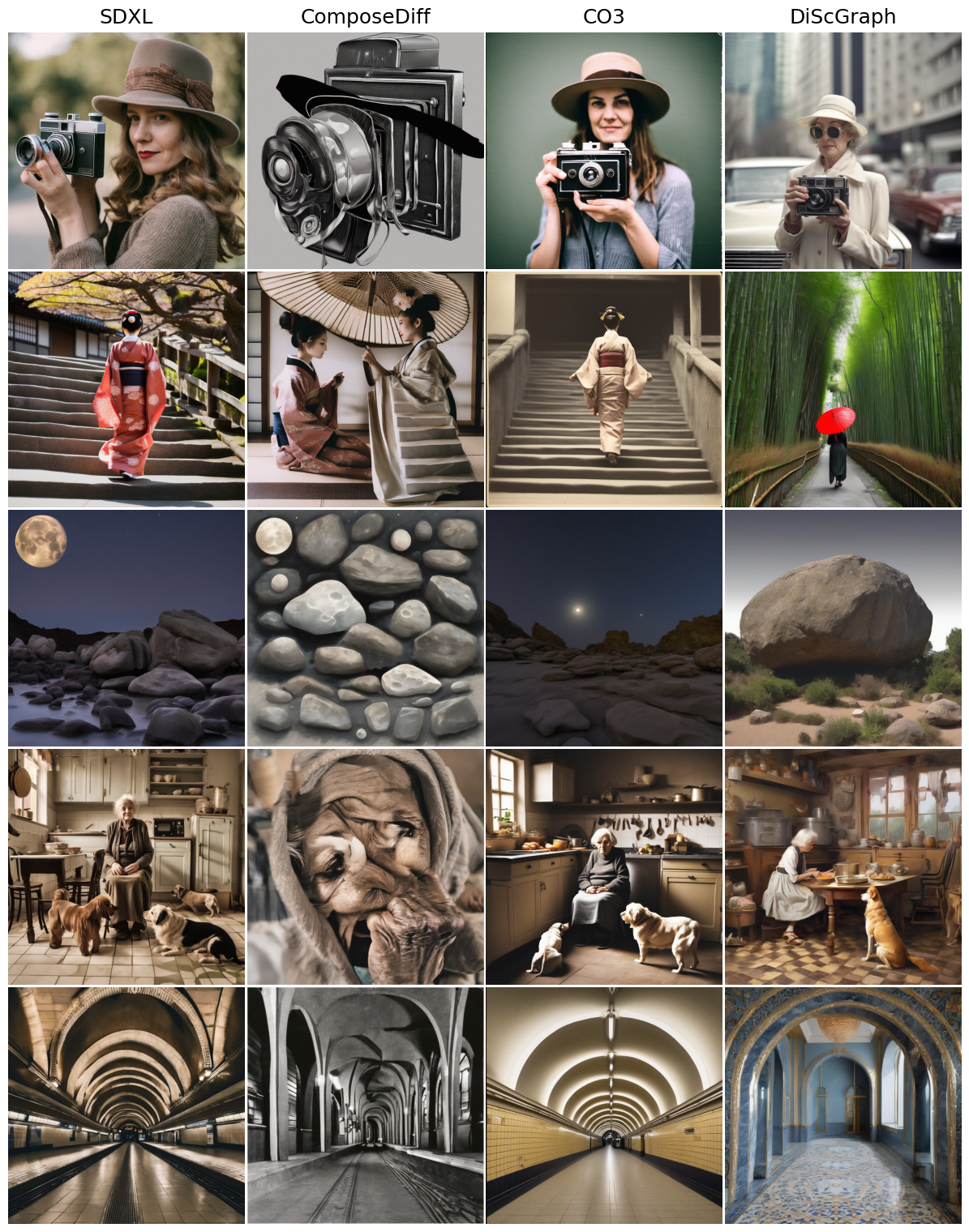}
    \caption{Qualitative Results: T2I generation comparison of methods, SDXL, ComposeDiff, CO3, and {\name}. For prompts:\textit{``An old woman in a white hat holding an old camera with a blurred car and buildings in the background.", ``A woman in traditional Japanese dress walking down a narrow path with trees. The woman's back is turned away.", ``A starry twilight sky and a rocky area with a boulder", ``An old kitchen that is worn out has large windows with frayed curtains on the right. The kitchen counter has many pots and pans, and cabinets above the stove. The kitchen floor has checkered pattern. A dog sits in front of a dining table and an old woman is sitting at the table cooking.", ``The interior of a cavernous building with many arches forming a pattern, golden and made of light blue ,marble."}}
    \label{fig:t2i_comp_2}
    \vspace{-1.7em}
\end{figure*}

\begin{figure*}[t]
    \centering
    \includegraphics[width=1\columnwidth]{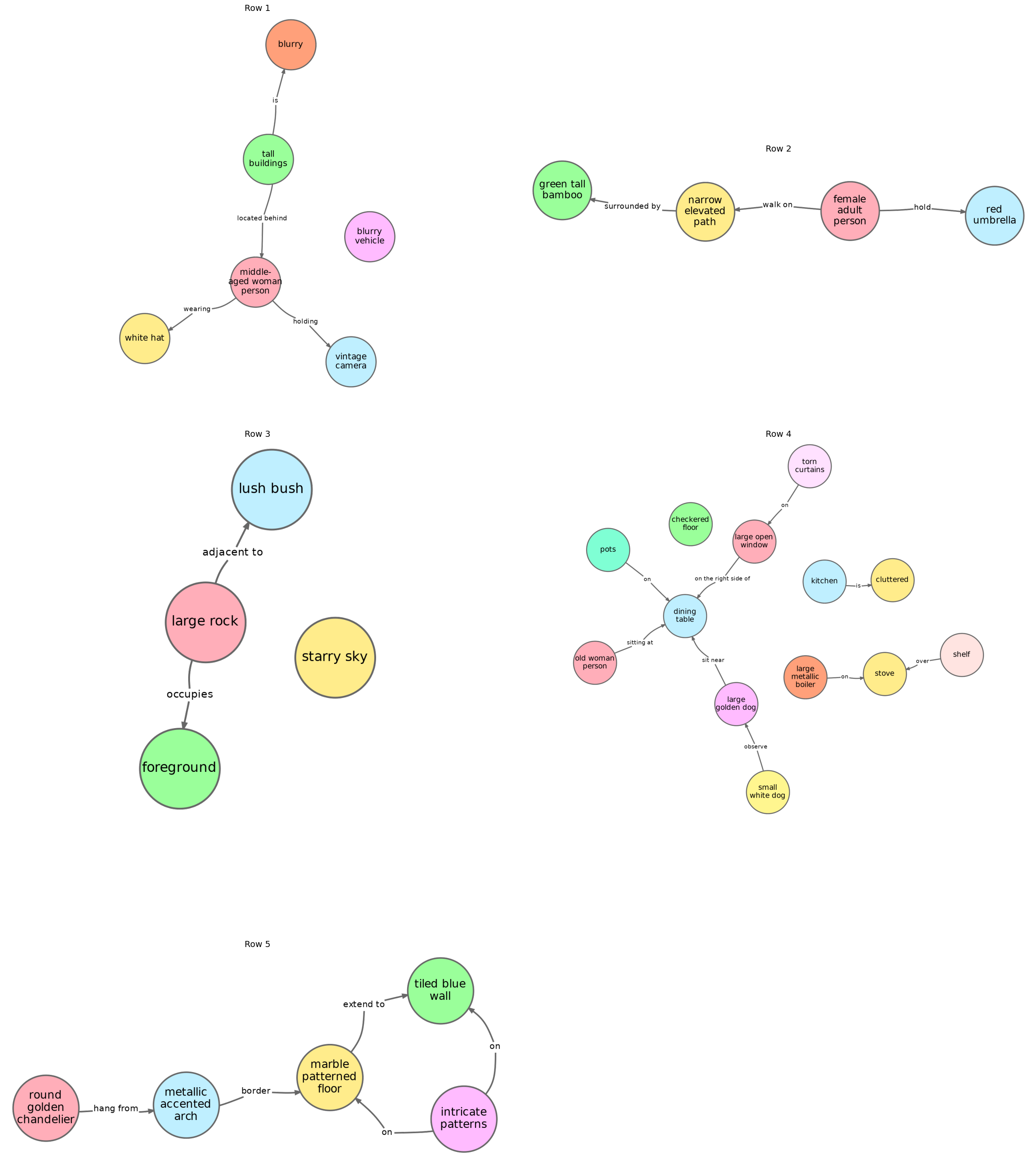}
    \caption{Qualitative Results: Sampled text-conditioned scene graphs used for generating image results of {\name} in Figure \ref{fig:t2i_comp_2}}
    \label{fig:t2i_comp_sg_2}
    \vspace{-1.7em}
\end{figure*}

\subsubsection{Text-to-SG-to-Layout to Image Generation}
We highlight {\name}'s capability to learn grounding signals from scene graphs in Figure \ref{fig:t2i_layout}. The text-conditioned scene graphs are converted to layouts by the layout head, and these layouts are used to generated images using GLIGEN \cite{li2023gligenopensetgroundedtexttoimage} model. {\name}'s derived layouts thus convert text into a plausible explanation of the visual scene and achieve performance comparable to layout generation methods such as LDM  \cite{lian2024llmgroundeddiffusionenhancingprompt}. LDM uses an LLM to infer layouts from text and uses the same downstream model GLIGEN to generate images. Thus, this result highlights that {\name} learns comparable grounding layouts derived from its good understanding of scene graphs. We specifically use prompts with multiple objects and attributes to highlight the performance.

\begin{figure*}[t]
    \centering
    \includegraphics[width=1\columnwidth]{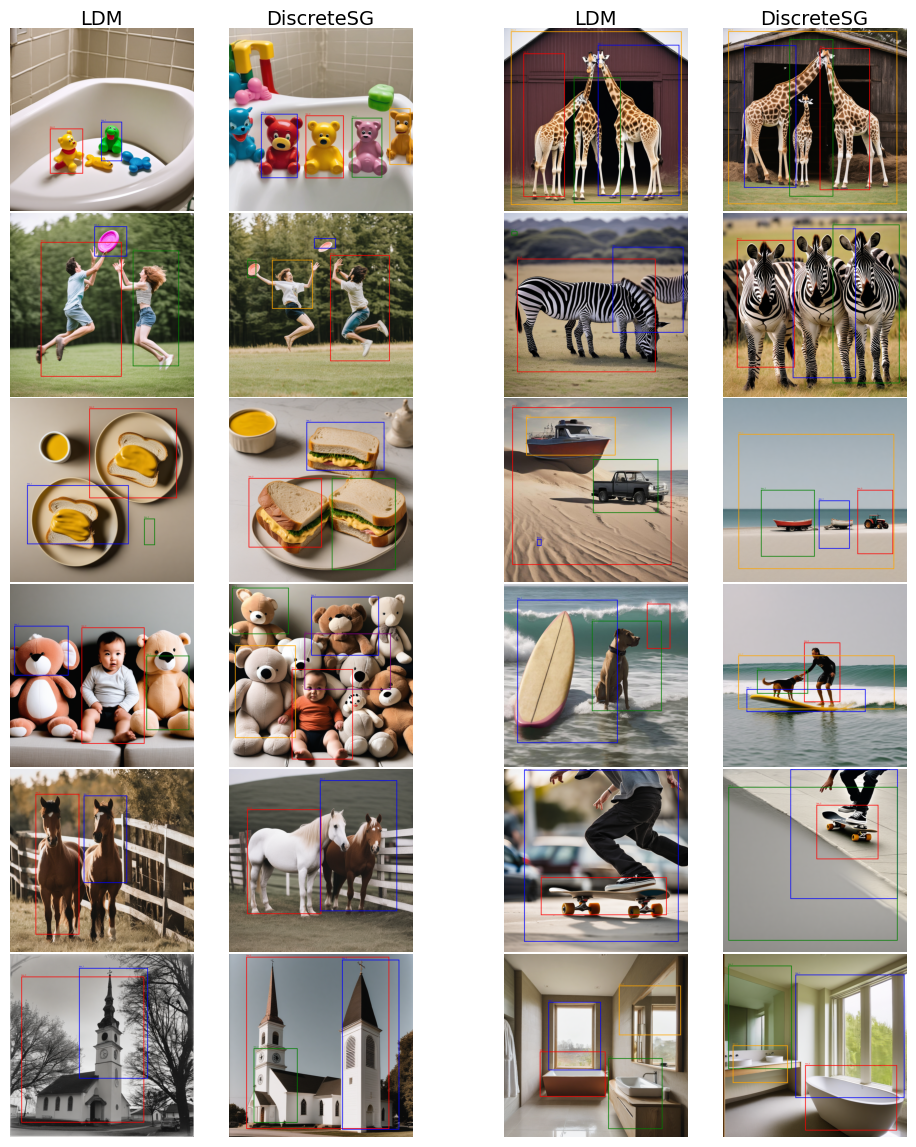}
    \caption{Qualitative Results:  Text -> SG -> Layout -> image results of LDM \cite{lian2024llmgroundeddiffusionenhancingprompt} and {\name}. For prompts:\textit{:``a bathtub with four plastic bear toys in bright colors", ``a man and a woman each throwing a frisbee while jumping in the air, ``three egg sandwiches on a white plate",``a baby surrounded by many bears", ``a white horse and a brown horse near a stable,``a church with a big turret and a clock tower surrounded by trees", ``three giraffes standing in front of a barn",``three zebras in grasslands",``a tractor and two boats on a white sand beach", ``a man and a dog on a surfboard riding a wave", ``a boy on a skateboard on a road", ``a bathroom with a big bath tub, a large window on the right, a big mirror above a sink."}}
    \label{fig:t2i_layout}
    \vspace{-1.7em}
\end{figure*}

\subsection{Broader Impacts}
 \label{sec:impact}
{\name} improves text-to-image generation by learning a structured scene graph prior using a factorized discrete diffusion model. At the same time, it introduces certain risks. It could be misused to create deceptive or misleading visuals, contributing to the spread of misinformation. When applied to depictions of public figures, it may compromise personal privacy. Moreover, the generated content can raise concerns around copyright and intellectual property.


\end{document}